\def\year{2022}\relax
\documentclass[letterpaper]{article} 
\usepackage{aaai22}  
\usepackage{times}  
\usepackage{helvet}  
\usepackage{courier}  
\usepackage[hyphens]{url}  
\usepackage{graphicx} 
\usepackage{natbib}  
\usepackage{caption} 
\DeclareCaptionStyle{ruled}{labelfont=normalfont,labelsep=colon,strut=off} 
\usepackage{algorithm}
\usepackage{algpseudocode}

\usepackage{newfloat}
\usepackage{listings}
\floatstyle{ruled}
\newfloat{listing}{tb}{lst}{}
\floatname{listing}{Listing}
\usepackage{booktabs}
\usepackage{amsfonts}       
\usepackage{nicefrac}       
\usepackage{microtype}      
\usepackage{xcolor}         

\usepackage{amsmath}       
\usepackage{mathtools}
\usepackage{amsthm}
\usepackage{ifthen}

\newtheorem{definition}{Definition}
\newtheorem{proposition}{Proposition}

\newtheorem{lemma}{Lemma}

\title{Chaining Value Functions for Off-Policy Learning}
\author{
    Simon Schmitt,\textsuperscript{\rm 1}\textsuperscript{\rm 2}
    John Shawe-Taylor,\textsuperscript{\rm 2}
    Hado van Hasselt\textsuperscript{\rm 1}
}
\affiliations{
    \textsuperscript{\rm 1}DeepMind \\
    \textsuperscript{\rm 2}University College London, UK \\
    suschmitt@google.com
}

\newcommand{\defeq}{\vcentcolon=}

\newcommand{\tk}[1]{#1_{t}^{k}}

\newcommand{\tpk}[1]{#1_{t+1}^{k}}

\newcommand{\tkm}[1]{#1_{t}^{k-1}}
\newcommand{\tkp}[1]{#1_{t}^{k+1}}

\renewcommand{\t}[1]{#1_{t}}
\newcommand{\tp}[1]{#1_{t+1}}

\renewcommand{\k}[1]{#1^{k}}
\newcommand{\km}[1]{#1^{k-1}}
\newcommand{\kp}[1]{#1^{k+1}}

\newcommand{\trans}[1]{{#1}^\top}

\newcommand{\thetatp}{\tp{\theta}}
\newcommand{\thetatpk}{\tpk{\theta}}
\newcommand{\thetat}{\t{\theta}}

\newcommand{\thetatkp}{\tkp{\theta}}
\newcommand{\thetatkm}{\tkm{\theta}}
\newcommand{\thetatk}{\tk{\theta}}
\newcommand{\thetakp}{\kp{\theta}}
\newcommand{\thetak}{\k{\theta}}
\newcommand{\thetakm}{\km{\theta}}
\newcommand{\thetapi}{\theta_\pi}

\newcommand{\st}{S_{t}}
\newcommand{\stp}{S_{t+1}}
\newcommand{\fst}{\phi(\st)}
\newcommand{\fstp}{\phi(\stp)}
\renewcommand{\vec}[1]{\mathbf{#1}}
\newcommand{\mat}[1]{\mathbf{#1}}

\newcommand{\tdb}{\vec{b}_{\pi}}
\newcommand{\tdbmu}{\vec{b}_{\mu}}

\newcommand{\tdA}{\mat{A}_{\pi}}
\newcommand{\tdAmu}{\mat{A}_{\mu}}
\newcommand{\tdX}{\mat{X}}
\newcommand{\tdY}{\mat{Y}}
\newcommand{\tdW}{\mat{W}}
\newcommand{\tdV}{\mat{C}}
\newcommand{\tdU}{\mat{U}}
\newcommand{\tdProj}{\mat{\Pi}}
\newcommand{\tdM}{\mat{M}}

\newcommand{\Dmu}{\mat{D}_\mu }

\newcommand{\Ppi}{\mat{P}_\pi }

\newcommand{\Expectation}[2][]{\mathbb{E}_{#1}\left[#2\right]}

\newcommand{\setNaturalToK}{\mathbb{Z} . k \leq K}

\newcommand{\norm}[1]{\|#1\|_2}

\newcommand{\tdthetas}{
\begin{bmatrix}
    \theta^0 \\ \theta^1 \\ \vdots \\ \theta^K
\end{bmatrix}
}
\newcommand{\tdMs}{
\begin{bmatrix}
    \tdAmu    &         &      & \dots  &   \vec{0} \\
    -\gamma\tdY   &  \tdX    &    &    & \vdots   \\
     & \ddots & \ddots  & &  \\
    \vdots     &       &   -\gamma\tdY    & \tdX &  \\
    \vec{0}       &   \dots  &   &   -\gamma\tdY    & \tdX
\end{bmatrix}
}
\newcommand{\tdbs}{
\begin{bmatrix}
    \tdbmu \\ \tdb \\ \vdots \\ \tdb
\end{bmatrix}
}

\newcommand{\Dmuroot}{\Dmu^{\frac{1}{2}}}
\newcommand{\Dmuinvroot}{\Dmu^{-\frac{1}{2}}}

\newcommand{\invXY}{\tdX^{-1}\tdY}
\newcommand{\invXYi}[1]{\left( \invXY  \right)^{#1}}
\newcommand{\invXgY}{\tdX^{-1}\tdY\gamma}
\newcommand{\invXgYi}[1]{\left( \invXgY  \right)^{#1}}

\newcommand{\OPE}{\mathrm{OPE}}

\newcommand{\spectral}[1]{\rho\left(#1\right)}

\newcommand{\kth}{$k^{\text{th}}$}
\newcommand{\kmth}{$(k-1)^{\text{th}}$}

\newcommand{\numberth}[1]{${#1}^{\text{th}}$}

\begin{document}

\maketitle

\begin{abstract}
To accumulate knowledge and improve its policy of behaviour, a reinforcement learning agent can learn `off-policy' about policies that differ from the policy used to generate its experience. This is important to learn counterfactuals, or because the experience was generated out of its own control. However, off-policy learning is non-trivial, and standard reinforcement-learning algorithms can be unstable and divergent.

In this paper we discuss a novel family of off-policy prediction algorithms which are convergent by construction. The idea is to first learn on-policy about the data-generating behaviour, and then bootstrap an off-policy value estimate on this on-policy estimate, thereby constructing a value estimate that is partially off-policy. This process can be repeated to build a chain of value functions, each time bootstrapping a new estimate on the previous estimate in the chain. Each step in the chain is stable and hence the complete algorithm is guaranteed to be stable. Under mild conditions this comes arbitrarily close to the off-policy TD solution when we increase the length of the chain. Hence it can compute the solution even in cases where off-policy TD diverges. 

We prove that the proposed scheme is convergent and corresponds to an iterative decomposition of the inverse key matrix. Furthermore it can be interpreted as estimating a novel objective -- that we call a `k-step expedition' -- of following the target policy for finitely many steps before continuing indefinitely with the behaviour policy. Empirically we evaluate the idea on challenging MDPs such as Baird's counter example and observe favourable results.
\end{abstract}

\noindent
Value estimation is key to decision making and reinforcement learning \citep{SuttonBarto:2018}.
To accumulate knowledge and improve its policy of behaviour, an agent can estimate values \emph{off-policy} corresponding to policies that differ from the policy used to generate the experience it learns from. This can be useful to learn counterfactuals, or because the experience was generated out of its own control. Indeed the applications of off-policy learning are manifold: learning to exploit while exploring as e.g. in $\epsilon$-greedy, learning multiple policies concurrently \citep{Sutton:2011,Badia:2020Never}, for representation shaping \citep{jaderberg:2016reinforcement}, to minimize costly mistakes \citep{Hausknecht:2000PlanningTreatment} or to learn from demonstrations~\citep{Hester:2018DeepDemonstrations}.

However, off-policy learning is non-trivial, because standard reinforcement-learning algorithms can be unstable: 
\citep{Baird:1995} showed that off-policy TD predictions can diverge to infinity in what is now known as \emph{Baird's MDP}. \citep{SuttonBarto:2018} attribute this to the popular combination of function approximation (to support large state spaces) and bootstrapping (to reduce variance) in the off-policy context since called the \emph{deadly triad}. Both are essential and ubiquitous in deep reinforcement learning ~\citep{Hasselt:2018Deadly} hence algorithms that are convergent even in the face of the deadly triad are a prominent research direction.

Over the years, several variants and solutions have been proposed \citep{Sutton:2009,Maei:2011,vanHasselt:2014,Sutton:2016}, but these do not uniformly outperform off-policy TD \citep{Hackman:2013} and sometimes suffer from high (even infinite) variance \citep{Sutton:2016}.

In this paper we analyze a novel family of off-policy prediction algorithms that is convergent (i.e. breaks the deadly triad) and conceptually simple.
The idea is to first learn on-policy about the data-generating behaviour, and then bootstrap an off-policy value estimate on this on-policy estimate, thereby constructing a value estimate that is partially off-policy.  This process can be repeated to build a chain of value functions, each time bootstrapping a new estimate on the previous estimate in the chain. Each step in the chain is stable and hence the complete algorithm is guaranteed to be stable.
When employing off-policy TD at each step in the chain we call it \emph{chained TD} learning.
While off-policy TD sometimes diverges and is unable to obtain its own solution (fixed point) we prove that chained TD always converges and that its solution comes arbitrarily close to the off-policy TD solution under mild conditions when we increase the length of the chain.

Interestingly our approach can be interpreted as estimating the value of following the target policy for a finite number of steps $k$ and then following the behaviour indefinitely.
We call this behaviour a 
k-step $\pi$-expedition (\emph{k-step expedition} in short) as the prediction envisions a $k$-steps limited `expedition' following a potentially novel $\pi$ before continuing with the well known behaviour $\mu$. Naturally longer and longer expeditions (larger $k$) approach the target policy. Chained TD exploits the recursive structure of this objective to reduce variance through bootstrapping.
For TD learning -- contrary to estimating the target value directly -- this is guaranteed to be stable as we prove in this paper.

While in practice we use a finite number of value functions we also consider what happens if $k \to \infty$ and use this to acquire insights into the convergence of the popular -- albeit different -- technique of \emph{target networks} \cite{Mnih:2015}.

We prove convergence of the expected chained TD update with a single learning rate and empirically confirm it on Baird's counter example that we augment to include rewards, where TD, TDC, GTD2 and ETD either diverge or make little progress.

\section{Background}
We consider state values $v(s)$ that are parameterised by parameter vector $\theta$---for instance the weights of a neural network. The goal is to approximate the value of each state $s$ under target policy $\pi$, as defined by
\begin{align*}
v_{\pi}(s)
& \defeq \Expectation{\sum_{i=0}^\infty \gamma^i R_{t+i+1} \mid S_t = s} \\
& = \Expectation{R_{t+1} + \gamma v(S_{t+1}) \mid S_t = s} \,.
\end{align*}
Off-policy TD \citep{SuttonBarto:2018} is an iterative process
\begin{equation}
\begin{split}
\thetatp &\defeq \thetat + \alpha \rho_t \left[R_t + \gamma v(\stp) - v(\st)\right] \nabla_{\thetat} v(\st)
\end{split}
\label{eq:stochastic_td}
\end{equation}
where each update aims to improve the parameters $\thetat$ such that the new estimate $v_{\thetatp}$ on average gets closer to the target value $v_\pi$, even when following a different policy $\mu$. Here $\alpha$ is the step-size, $\gamma$ is the discount and $R_t$ is the reward observed when transitioning from state $\st$ to $\stp$ after executing action $A_t \sim \mu(A_t | \st)$.
In update \eqref{eq:stochastic_td}, $\rho_t \defeq \pi(A_t | \st) / \mu(A_t | \st)$ is the \emph{importance-sampling ratio} between the probability of selecting action $A_t$ under the target policy $\pi$ and under the behaviour policy $\mu$ -- not to be confused with the spectral radius
of a matrix $\spectral{\mat{M}}$.
Unfortunately, when using function approximation, convergence of this algorithm can only be guaranteed in the on-policy setting where $\pi=\mu$ \citep{Baird:1995,SuttonBarto:2018}.

This is an actively pursued research area where
a series of solutions have been proposed \citep{Sutton:2009,Maei:2011,vanHasselt:2014,Sutton:2016}, but these often suffer from either performing worse than off-policy TD when it does not diverge \citep{Hackman:2013} or even from infinite variance \citep{Sutton:2016}. 
Our approach is similar in spirit to \citep{DeAsis:2020FixedHorizon} that estimate a new kind of return: fixed horizon returns (i.e. the rewards only from the next $k$ steps) instead of the typical discounted return. This special return can also be estimated through a series of value functions and is guaranteed to converge albeit to a different fixed point.
The special case of chaining for a single step has been considered before: \citep{wiering2007qv} consider bootstrapping an action value off of a state value, which itself is learnt on-policy or off-policy \citep{wiering2009qvfamily}. \citep{mazoure2021improving} consider bootstrapping off of an on-policy estimate with an off-policy multi-step return. These approaches can all be interpreted as performing one step in the more general chained TD algorithms that we consider in this paper.

\section{Chaining Off-Policy Predictors}
We want an off-policy algorithm that is 1) stable (i.e., convergent) and 2) with low bias with respect to the true values $v_{\pi}$.
To this extend we propose a novel family of algorithms and show that it satisfies these desiderata.

\begin{algorithm}[tb]
\caption{\textbf{Sequential chained TD} is described below. \textbf{Concurrent chained TD} is obtained by moving line 2 between line 6 and 7. Note that $T$ needs to be specified large enough to ensure convergence.}
\label{alg:algorithm}
\textbf{Input}: $\pi$, $\mu$, number of chains $K$, number of update steps $T$\\
\textbf{Parameter}: step size $\alpha$
\begin{algorithmic}[1]
\State Initialize all $\{\thetak\}_{k \in \setNaturalToK}$ randomly, $t\gets 0$.
\For{$k \gets 0$ to K}
\For{$i \gets 1$ to T}
\State $t \gets t + 1$
\State Play one action $A_t$ with $\mu$.
\State Observe next state $\stp$ and reward $\tp{R}$.
\If {$k=0$}
\State $\delta \gets \tp{R} + \gamma v_{\theta^0}(\stp) - v_{\theta^0}(\st)$; $\rho \gets 1$
\Else 
\State $\delta \gets \tp{R} + \gamma v_{\thetakm} (\stp) - v_{\thetak}(\st)$
\State $\rho \gets \frac{\pi(A_t| \st)}{\mu(A_t| \st)}$
\EndIf
\State $\thetak \gets \thetak + \alpha \rho \delta \nabla_{\theta} \k{v}(\st)$
\EndFor
\State $\thetakp \gets \thetak$ \Comment{Only used in sequential chained TD.}
\EndFor
\State \textbf{return} $\{\thetatk\}_{k \in \setNaturalToK}$
\end{algorithmic}
\end{algorithm}

\newcommand{\NaturalAndZero}{\mathbb{N}_0}
Starting with the behaviour value $$v^0 \defeq v_\mu$$
the idea is to define a series of value functions $\{v^k\}_{k \in \NaturalAndZero}$ recursively such that they approach the desired target value:
$$ \lim_{k \to \infty} v^k \to v_{\pi} $$
This is achieved recursively by employing an off-policy estimator $\OPE$ such as off-policy TD learning that estimates $\k{v}$ by bootstrapping off the previous value $\km{v}$:
$$ \k{v} \defeq \Expectation[\tau \sim \mu]{ \OPE(\pi, \km{v}, \tau, \mu)} $$
This principle can be applied to any off-policy estimator that employs trajectories $\tau$ sampled from $\mu$ and a bootstrap value $\km{v}$ to predict the values of target policy $\pi$ e.g. 
$ \k{v}(s) \defeq \Expectation[\tau \sim \mu]{\rho_t(R_t + \gamma \km{v}(\stp))  \mid  \st = s} $.

The idea of chaining off-policy estimators has a natural interpretation: $v_\mu$ is the value of the behaviour policy $\mu$ and $v^k$ has the value of at first performing $k$ steps according to the target policy $\pi$ and then following $\mu$ indefinitely. We call such behaviour an \emph{k-step expedition and $v^k$ the \emph{k-step expedition value}.}
\begin{definition}
A \emph{k-step expedition} from state $s$ acts with $\pi$ for k steps and then with $\mu$ indefinitely. Let the \emph{k-step expedition value} of state $s$ be the expected return of a k-step expedition from state $s$.
\end{definition}
As $k$ increases the value $v^k$ becomes more and more off-policy and $v^0$ ultimately becomes irrelevant. This perspective illustrates that typically $v^k \to v_\pi$ as $k$ increases (i.e. that $v^k$ becomes unbiased). While this is easy to see for tabular RL, we analyse bias and convergence in the more general case of function approximation in the following section.

If estimated \emph{sequentially} convergence is guaranteed by induction: $v^0 \defeq v_\mu$ can be estimated on-policy, and hence TD is \emph{stable} (i.e. converges). Then, for each $k > 0$, $v^{k}$ is stable because it bootstraps off a stable $v^{k-1}$. In the next section we prove that convergence is also guaranteed for chained off-policy learning if all parameters are updated \emph{concurrently}, e.g., when learning all value functions online.

A concrete stochastic update for each such value function when transitioning from state $\st$ to $\stp$ and observing a reward $\tp{R}$ is given by
\begin{equation}\label{eq:alg}
\thetatpk \defeq \thetatk + \alpha \rho_t \tk{\delta} \nabla_{\thetatk} \tk{v}(\st) \,,
\end{equation}
where $\thetatk$ are the parameters of the \kth value function after observing $t$ transitions
, $\tk{v}(s) \defeq v_{\thetatk}(s)$
and
\[
    \tk{\delta} \defeq \tp{R} + \gamma \tkm{v}(\stp) - \tk{v}(\st) \,.
\]
We call this \emph{chained (off-policy) TD} learning. \emph{Sequential chained TD} only updates $\thetak$ on timesteps after the previous $\thetakm$ has converged, while \emph{concurrent chained TD} updates all $\{\thetak\}_{k}$ at each timestep (see Algorithm~\ref{alg:algorithm}).

In the next sections we analyse both algorithms
theoretically (Section \ref{sec:theory}) and empirically (Section \ref{sec:experiments}).

\section{Analysis}\label{sec:theory}
To analyse Algorithm~\ref{alg:algorithm} in this section we consider linear function approximation, so that $v_{\theta}(s) = \theta^{\top} \fst$, where $\fst$ are the \emph{features} observed at time $t$. We recall that off-policy TD sometimes diverges and is unable to obtain its own solution (fixed point) $\thetapi\defeq\tdA^{-1}\tdb$, then we show that chained TD is always convergent and can compute $\thetapi$ under mild conditions via the following steps:

\begin{enumerate}
    \item Section~\ref{sec:fixed_point_recursion} observes
    that sequentially chained TD defines a recursion of fixed points: The fixed point $\thetak_*$ of value function $v^k$ can be computed from $\thetakm_*$.
    \item Section~\ref{sec:bias} shows that this recursion approaches the off-policy solution under mild conditions: $\lim_{k\to\infty}\thetak_* = \thetapi$.
    \item Section~\ref{sec:convergence_chained} proves convergence of both expected sequential and concurrent chained TD to the fixed points: $\lim_{t\to\infty}\thetatk = \thetak_*$.
\end{enumerate}
Hence chained TD is convergent for any fixed $k$ and the attained fixed points of the \kth value function $\thetak_*$ indeed approaches the off-policy TD solution $\thetapi \defeq \tdA^{-1} \tdb$ under mild conditions that we investigate further in~\ref{sec:bias}. Then chained TD learning is unbiased wrt. $\thetapi$ in the limit: i.e. $\lim_{k \to \infty} \thetak_* = \thetapi$.

Off-policy TD and chained TD can be analysed through their expected updates which can be written in matrix form. For off-policy TD (Equation~\eqref{eq:stochastic_td}) we obtain:
\begin{equation}
 \thetatp = \thetat + \alpha \left( \tdb - \tdA \thetat \right)
    \label{eq:iterative_preimage}
\end{equation}
and the expected update for chained TD
(Equation~\eqref{eq:alg})
is
\begin{equation}
 \thetatpk
  =\thetatk + \alpha \left( \tdb + \gamma\tdY \thetatkm - \tdX \thetatk \right) \,.
\label{eq:chain_matrix_update}
\end{equation}
with
\begin{align}
\tdb &\defeq \Expectation[\mu]{ \rho_t R_t \fst }, \tdbmu \defeq \Expectation[\mu]{ R_t \fst }, \\
\tdA &\defeq 
    \Expectation[\mu]{ \rho_t \fst \left(\trans{\fst} - \gamma \trans{\fstp} \right)} \\
    &= \trans{\Phi} \Dmu (I - \gamma \Ppi) \Phi = \tdX - \gamma\tdY\\
\tdX &\defeq \Expectation[\mu]{ \rho_t \fst \trans{\fst}} = \trans{\Phi} \Dmu \Phi \\
\tdY &\defeq \Expectation[\mu]{ \rho_t \fst \trans{\fstp}} = \trans{\Phi} \Dmu \Ppi \Phi \\
    \tdProj &\defeq \Phi \left(\trans{\Phi} \Dmu \Phi \right)^{-1}  \trans{\Phi} \Dmu = \Phi \tdX^{-1}  \trans{\Phi} \Dmu
\end{align}
where $\Phi$ is the state-feature matrix, $\Ppi$ is $\pi$'s transition matrix and $\Dmu$ is a diagonal matrix with $\mu$'s steady-state distribution, $\tdA$ is called the \emph{key matrix} and $\tdProj$ is called the \emph{projection matrix}~\cite{SuttonBarto:2018}. We make the common technical assumptions that the columns of $\Phi$ are linearly independent and that $\mu$ covers all states such that $\Dmu$ and hence $\tdX$ have full rank~\cite{Sutton:2016}.

\subsection{Viewing Expected TD as Richardson Iteration}
Expected TD (see equation~\eqref{eq:iterative_preimage}) can be viewed as Richardson Iteration \citep{Richardson:1911} which is a simple and well-studied iterative algorithm that given $\tdM$ and $\vec{b}$ converges to $\theta^{*}=\tdM^{-1}\vec{b}$ under the condition that all eigenvalues of $\tdM$ are positive. Rather than inverting $\tdA$ expected TD learning attempts to determine the solution of $\tdA\thetapi = \tdb$ iteratively through Richardson Iteration and may diverge even though $\tdA$ is invertible.

\begin{definition}
Given a square matrix $\tdM$, vector $\vec{b}$ and step-size $\alpha$ Richardson Iteration computes:
\begin{equation}
\thetatp = \thetat + \alpha \left( \vec{b} - \tdM \thetat \right)
\label{eq:richardson_iteration}
\end{equation}
\end{definition}

\begin{definition}
We call Richardson Iteration \emph{stable} if $\lim_{t\to \infty} \thetat$ converges.
\end{definition}
\begin{proposition}
\label{prop:richardson_convergence}
Let $\theta_1$ be any initial value, $\tdM$ any square matrix with only positive eigenvalues and $\vec{b}$ any compatibly shaped vector then Richardson Iteration
$ \thetat$
converges to $\theta^{*} = \tdM^{-1} \vec{b}$ for a sufficiently small step size $\alpha$.
\end{proposition}

\begin{proof}
Let $\t{r} = \theta_t - \theta^*$, then
\begin{equation}
\begin{split}
\tp{r} &= \theta_t + \alpha \left( \vec{b} - \tdM \thetat \right)- \theta^* 
    = \theta_t + \alpha \left( \tdM \theta^* - \tdM \thetat \right)- \theta^* \\
    &= \left(I - \alpha \tdM \right) \t{r} 
    = \left(I - \alpha \tdM \right)^t r_0
\end{split}
\end{equation}
Since $\tdM$ has only positive eigenvalues we can pick $\alpha$ such that $I - \alpha \tdM$ satisfies $|\lambda_i| < 1.0$ for all eigenvalues $\lambda_i$. Furthermore we can diagonalize $I - \alpha \tdM = \mat{V} \mat{\Lambda} \mat{V}^{-1}$ such that $\left(I - \alpha \tdM \right)^k = \mat{V} \mat{\Lambda}^t \mat{V}^{-1}$. Since all entries of $\mat{\Lambda}$ have absolute value smaller than 1.0 convergence is ensured $\norm{\theta_t - \theta} = \norm{\t{r}} \to 0$ for $t \to \infty$ and any $\vec{b}$.
\end{proof}

\subsection{Fixed Point Recursion}
\label{sec:fixed_point_recursion}
The expected update of sequential chained TD~\eqref{eq:chain_matrix_update} can also be seen as Richardson Iteration. Once the \kmth value function is estimated and $\thetakm$ is fixed, the chained TD update for the next value and its parameters $\thetak$ converges to a fixed point $\thetak_*$ that depends on $\thetakm$:
\begin{equation}
    \thetak_*(\thetakm) \defeq \lim_{t \to\infty} \thetatk = \tdX^{-1}(\gamma\tdY \thetakm +\tdb)
    \label{eq:chain_recursion}
\end{equation}
convergence follows by Proposition~\ref{prop:richardson_convergence} for a sufficiently small step-size $\alpha$ because $\tdX$ is positive-definite. Should $\thetak$ bootstrap on the fixed point of a previous value $\thetakm_*$, we obtain a recursion of fixed points:
\begin{equation}
    \thetak_* = \tdX^{-1}(\gamma\tdY \thetakm_* +\tdb)
    \label{eq:chain_fixpoint_recursion}
\end{equation}

\subsection{Bias}
\label{sec:bias}
The established fixed point recursion~\eqref{eq:chain_fixpoint_recursion} can be interpreted as a transformation of the unstable off-policy TD inverse problem ("determine $ \thetapi=\tdA^{-1}\tdb $") where Richardson Iteration and hence TD diverge 
into a recursive sequence of stable sub-problems ("given $ \thetakm_*$ determine $ \thetak_*$") that are all stable under Richardson Iteration (see sections~\ref{sec:fixed_point_recursion} and~\ref{sec:convergence_chained}).
In this section we prove under which conditions
$$ \lim_{k \to \infty} \thetak_* = \tdA^{-1}\tdb$$
i.e. that the sequence of fixed points converges to the off-policy TD solution $\thetapi$ as $k$ increases.

\begin{proposition}
\label{prop:biased}
Let $\thetak_*$ denote the fixed point of the $k^{\text{th}}$ chained value function defined as Eq.~\eqref{eq:chain_fixpoint_recursion}. Its bias (distance to the TD off-policy solution $\thetapi\defeq\tdA^{-1} \tdb$) is then given by $\thetak_*-\thetapi=\gamma^k\invXYi{k}\left(\theta^0 - \tdA^{-1} \tdb  \right)$ for any initial value $\theta^0$. 
\end{proposition}
\begin{proof}
Given any $\theta^0$ (e.g. without loss of generality the fixed point $\theta_\mu$ of the on-policy algorithm estimating $v_\mu$), the sequence~\eqref{eq:chain_fixpoint_recursion} can be written in closed form as:
\begin{equation}
    \thetak_* = \underbrace{\sum_{i=0}^{k-1}\invXgYi{i}}_{\tdW_k} \tdX^{-1}\tdb + \invXgYi{k}\theta^0
    \label{eq:theta_closed_form}
\end{equation}
Since $\tdW_k$ is a geometric series wrt. $\invXgY$ it satisfies:
\begin{align*}
\mat{I} - \invXgYi{k} 
    & =\underbrace{\sum_{i=0}^{k-1}\invXgYi{i} }_{\tdW_k} \left( \mat{I} - \invXgY \right) \\
    &= \tdW_k \tdX^{-1}\left(\tdX-\gamma\tdY\right) \\
    &= \tdW_k \tdX^{-1} \tdA
\end{align*}
Hence 
\begin{equation}
    \tdW_k \tdX^{-1} = \tdA^{-1} - \invXgYi{k}\tdA^{-1}
\label{eq:wk_closed_form}
\end{equation}
Plugging this into the closed form for $\thetak_*$ from equation~\eqref{eq:theta_closed_form}:
\begin{align}
\thetak_*
    &= \tdW_k \tdX^{-1}\tdb + \invXgYi{k}\theta^0 \nonumber \\ 
    &= \tdA^{-1} \tdb - \invXgYi{k}\tdA^{-1} \tdb + \invXgYi{k}\theta^0 \nonumber \\ 
&= \tdA^{-1} \tdb + \underbrace{\gamma^k\invXYi{k}\left(\theta^0 - \tdA^{-1} \tdb  \right)}_{\mathbf{Bias\ wrt. \ \boldsymbol{\thetapi}}} 
\label{eq:theta_bias}
\end{align}
\end{proof}
Observe that $\invXgYi{k}$ can be rewritten in terms of the TD projection $\tdProj$ and the transition matrix $\Ppi$.
\begin{align}
\invXgYi{k} 
    &= \underbrace{ \tdX^{-1} \gamma \trans{\Phi} \Dmu \Ppi}_{\defeq\tdV} \left(\gamma 
    \underbrace{\Phi \tdX^{-1}  \trans{\Phi} \Dmu}_{=\tdProj} \Ppi \right)^{k-1} \Phi \nonumber \\
    &= \tdV \left(\gamma \tdProj \Ppi \right)^{k-1} \Phi
\label{eq:matrix_zp}
\end{align}

While we will see in the next section that chained TD is always convergent for any fixed $k$, Proposition~\ref{prop:biased} allows us to analyze its distance to $\thetapi$. For a fixed $k$ the distance depends on $\theta^0$. 
The distance can be greatly reduced should $\theta^0$ already be close to the solution $\tdA^{-1}\tdb$. Hence a heuristic choice of $\theta^0$ may be beneficial and without loss of generality we chose to use the behaviour value $v_{\mu}$ i.e. $\theta^0=\tdAmu^{-1} \tdbmu$ which is always convergent independently of $\pi$ and has recently been advocated with a single greedification step for offline RL \citep{gulcehre2021addressing,brandfonbrener2021offline}.

\subsubsection{Bias for Infinitely Long Chains ($\mathbf{k\to\infty}$)}
In practice we can only use chains of finite length but we can analyse what happens as the chains get longer. Below we prove that $\thetak_* \to \thetapi$ if $\spectral{\gamma\Ppi\tdProj}<1$.

Here $\tdProj$ is the TD projection and $\Ppi$ the transition matrix.
We can observe that $\spectral{\tdProj}\leq1$ and $\spectral{\Ppi} \leq 1$ hold for any MDP (see appendix). While those are not sufficient conditions to ensure that also $\spectral{\tdProj\Ppi}\leq1$ in practice it often still holds. In the appendix we conjecture and discuss why.
In Figure~\ref{fig:sucess_rate} we investigate this condition numerically: We show how often the provably convergent chained TD is unbiased in the limit of infinite $k$ on random MDPs and observe that it is nearly always the case. On the other hand off-policy TD on the same MDPs diverges in roughly $20\%$ of the cases.

\begin{proposition}
\label{prop:unbiased}
Let $\thetak_*$ denote the fixed point of the $k^{\text{th}}$ chained value function defined as Eq.~\eqref{eq:chain_fixpoint_recursion}. Then the fixed point limit $\theta^\infty_* \defeq \lim_{k\to\infty} \thetak_*$ is equal to $\thetapi\defeq\tdA^{-1} \tdb$ (i.e. $\theta^\infty_*=\thetapi$) for any initial value $\theta^0$ if either $\spectral{\invXgY}<1$ or equivalently $\spectral{\gamma\tdProj\Ppi}<1$.
\end{proposition}
\begin{proof}
$\spectral{\invXgY}<1 \implies \lim_{k \to \infty} \norm{\invXgYi{k}} = 0$ hence the bias in Proposition~\ref{prop:biased} vanishes as $k \to \infty$. By similar argument from $\spectral{\gamma\tdProj\Ppi}<1$ it follows that $\left(\gamma \tdProj \Ppi \right)^{k-1}$ converges to the zero matrix as $k\to\infty$. Then by Equation~\eqref{eq:matrix_zp} so does $\invXgYi{k}$.
\end{proof}
Hence besides being always convergent (see next section), chained TD can even be unbiased wrt. $\thetapi$ if $\spectral{\gamma\tdProj\Ppi}<1$. In that case the bias in Eq.~\eqref{eq:theta_bias} reduces exponentially with $k$. 
 
\begin{figure}[t]
\centering
\includegraphics[width=0.45\textwidth]{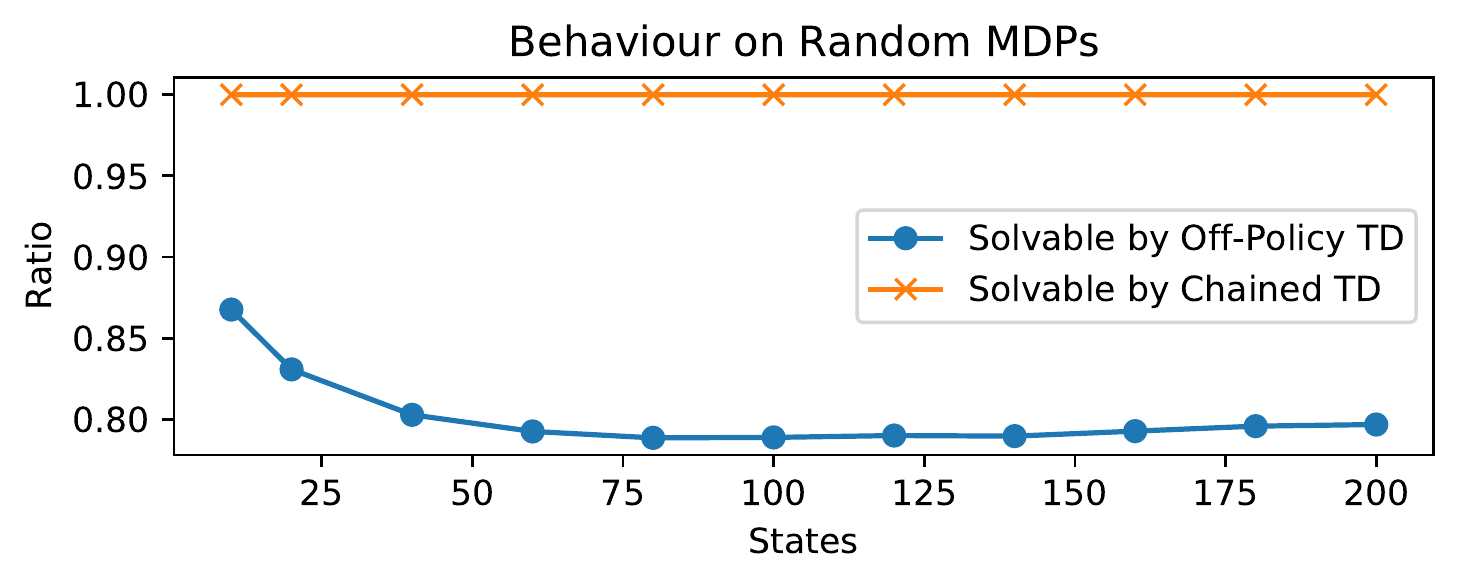}
\caption{
Off-policy TD sometimes diverges and is unable to obtain its own solution (fixed point). On the other hand chained TD always converges and often 
solves the off-policy TD problem to arbitrary precision for sufficiently large $k$
-- even in cases where off-policy TD diverges such as Baird's MDP.
The difference between off-policy TD and chained TD becomes most apparent by looking at their worst case scenarios: Off-policy TD diverges for MDPs such as Baird's or the two-state MDP \cite{Tsitsiklis:1997,Sutton:2016}. Chained TD obtains the target value in both MDPs. The latter can be modified such that chained TD becomes biased (see Section C of the appendix). However chained TD remains convergent i.e never diverges as we have proved.
One may now ask how often each algorithm is able to solve the TD problem to arbitrary precision.
We investigate this numerically by checking their relevant matrices $\tdA$ and $\tdProj\Ppi$ on random MDPs (sampling entries in $\Phi$ normal, rows of $\Ppi$ and the diagonal of $\Dmu$ uniformly and re-normalizing to sum to $1$) with $\gamma=0.99$ and as many features as states. We observe that chained TD solves the TD problem in nearly all cases while off-policy TD diverges in about $20\%$ of the cases.
Note that chained TD remains stable even when it does not solve the problem. Hence one may argue that the worst case scenario of chained TD is favourable.
}
\label{fig:sucess_rate}
\end{figure}

\subsection{Convergence}
\begin{figure*}[t]
\centering
\includegraphics[width=0.3\textwidth]{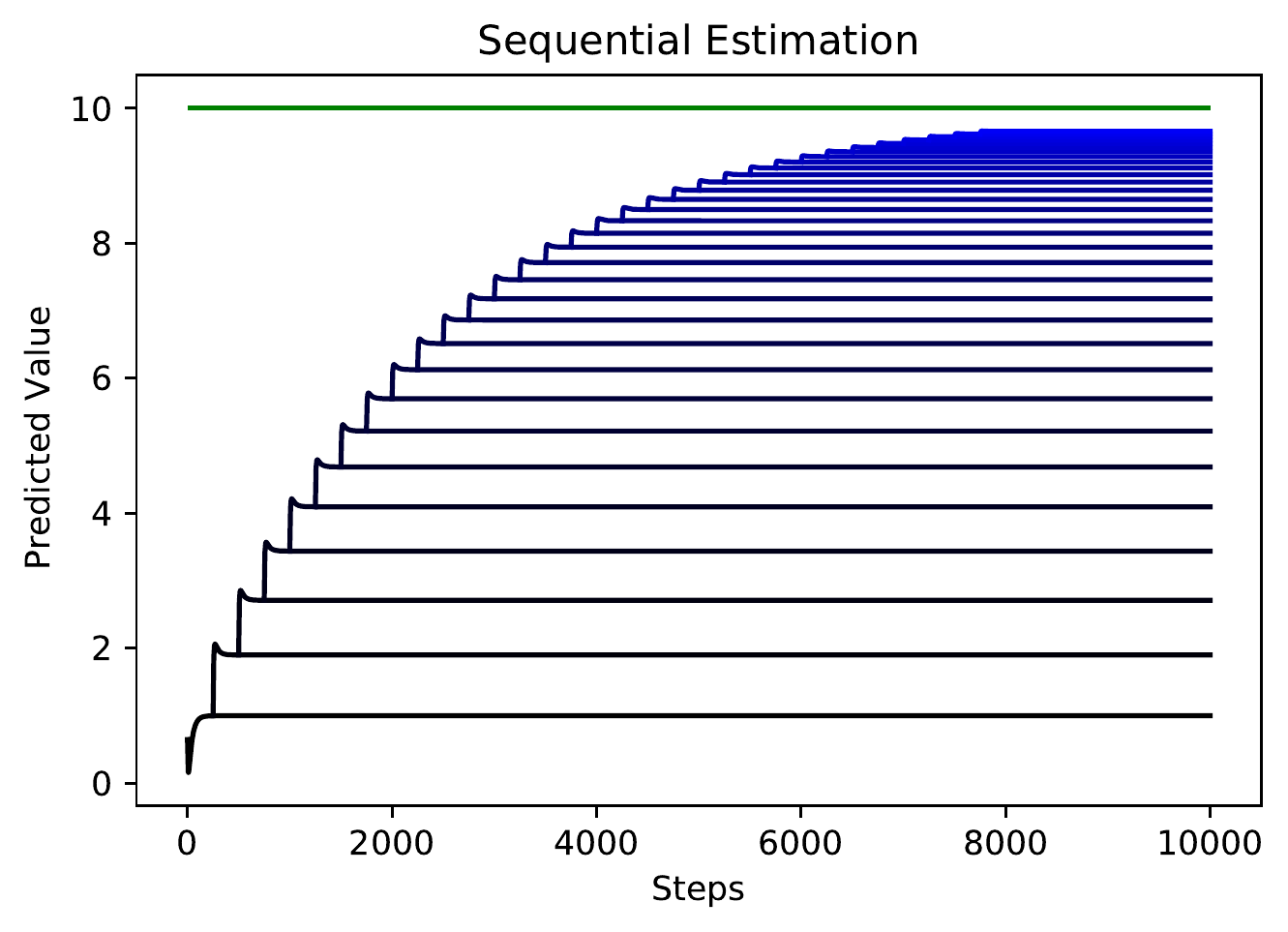}
\includegraphics[width=0.3\textwidth]{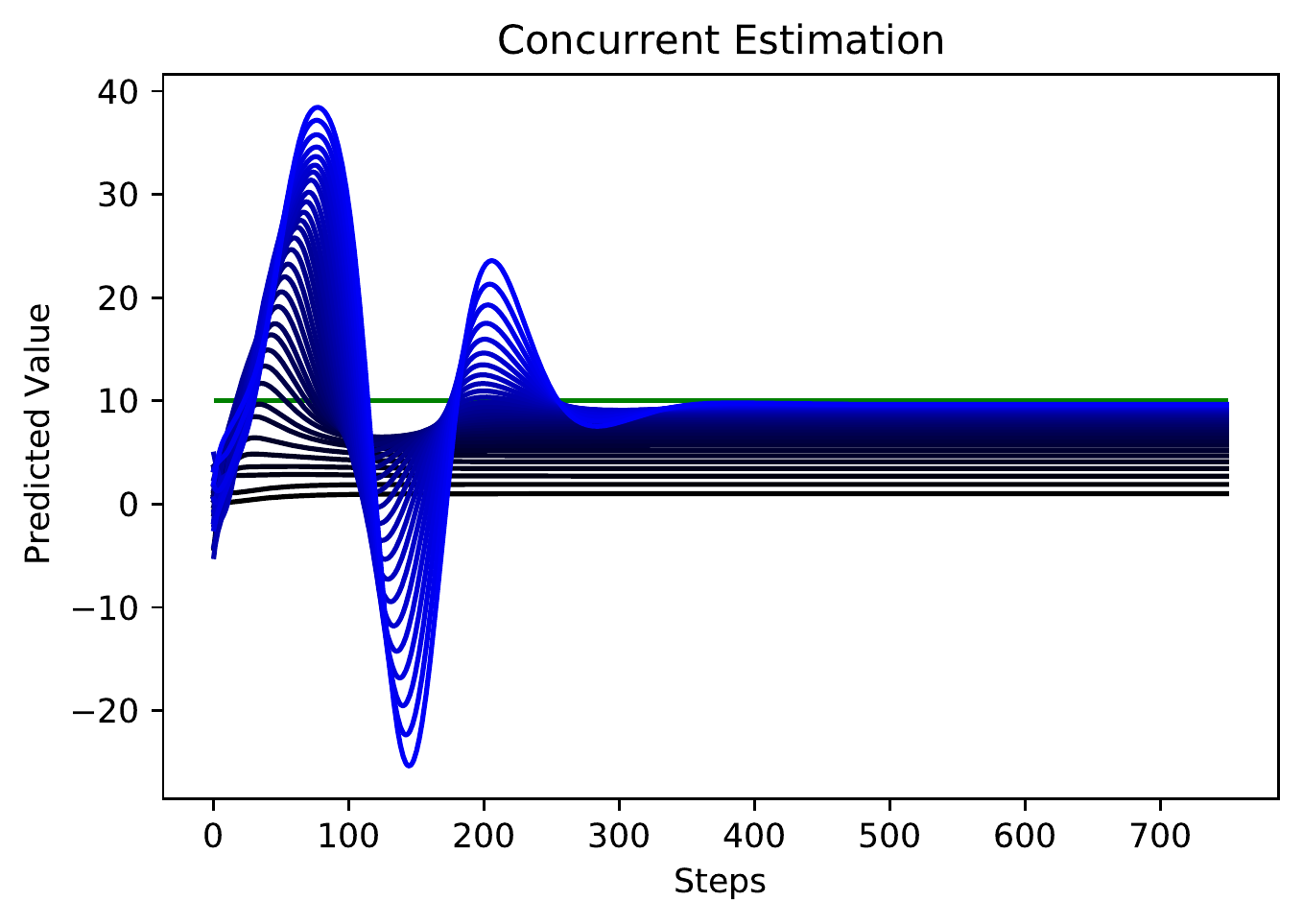}
\includegraphics[width=0.3\textwidth]{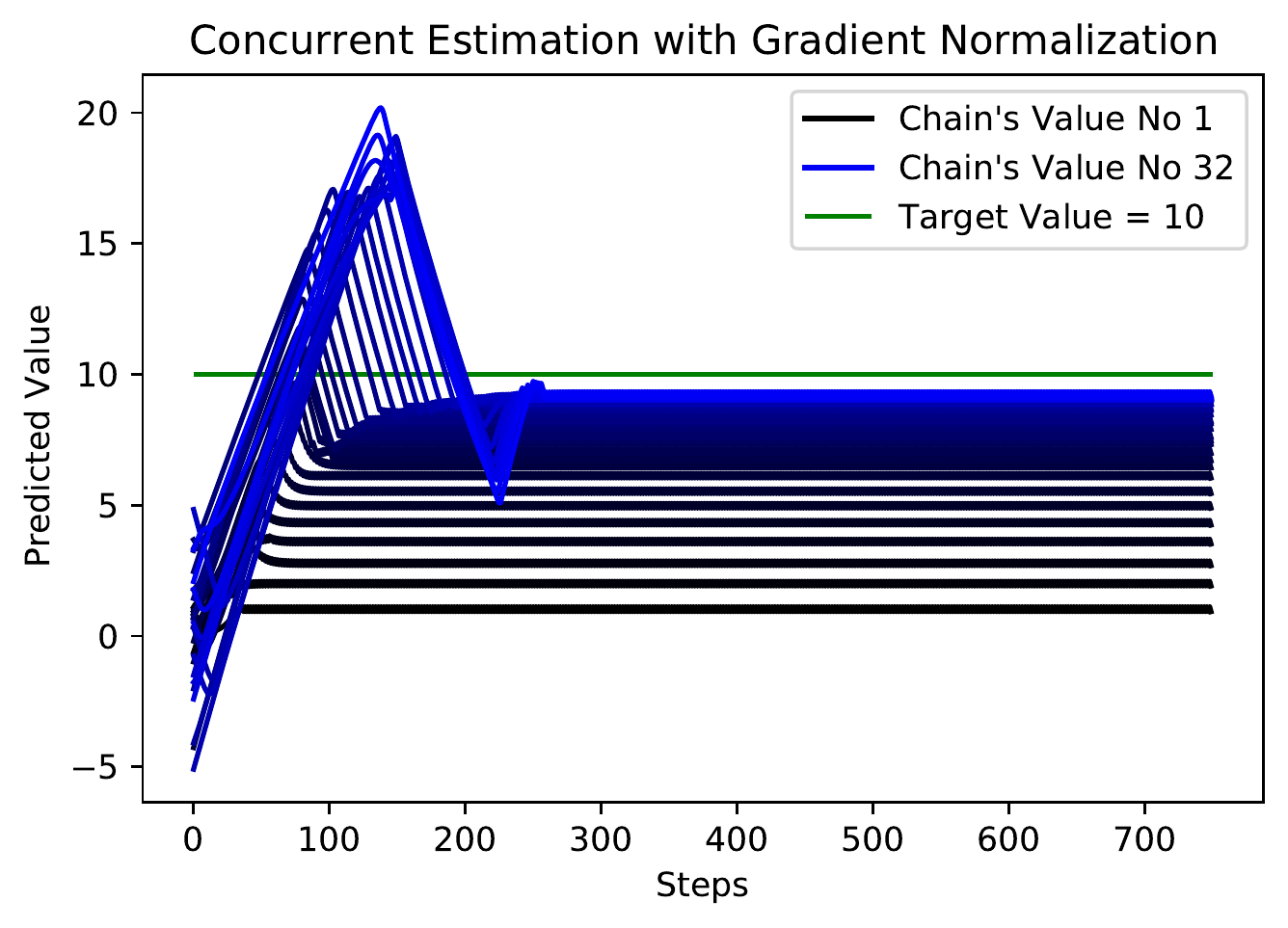}
\caption{Various implementations (all with step-size $\alpha=0.1$) of chained off-policy TD on Baird-Reward with discount $0.9$ evaluated at state 8. Note that the target value at any state is $1/(1-\gamma)=10$ and that all three displayed implementations approach the target-value as $k$ increases. \textbf{Left:} Sequential Estimation. \textbf{Center:} Observe how Concurrent Estimation converges to the same correct results as Sequential Estimation with faster pace but with oscillations prior to reaching the target value. \textbf{Right:} Concurrent Estimation with gradient normalization. Note that the oscillations are reduced and that the predictions approach the target value.}
\label{fig:expected_update}
\end{figure*}

\label{sec:convergence_chained}
The previous section showed when the unstable off-policy-TD inverse problem
$ \thetapi=\tdA^{-1}\tdb $
can be decomposed
into a recursive sequence of sub-problems ("given $\thetakm_*$ determine $ \thetak_*$")
that approach the off-policy TD solution ($\lim_{k\to \infty} \thetak_* = \thetapi$).
We will now show that each $\thetak_*$ can be estimated through TD learning. To do this we prove that the corresponding Richardson Iterations converge. Later we will show that all $\thetak_*$ can be determined concurrently, hence we do not need to wait until $\thetatkm$ has converged before updating $\thetatk$.

\paragraph{Sequential Estimation}

We call \emph{Sequential Estimation} the process where each value function $\k{v}$ bootstraps off the previous value function $\km{v}$ only when the latter has converged. The resulting $\thetakm_*$ is then fixed and used as a TD bootstrap target in Equation~\eqref{eq:chain_matrix_update} to estimate the next $\thetak_*$. Convergence can be proved by induction. Given a convergent initial value e.g. $\theta^0_* \defeq \theta_\mu$ or previous solution $\thetakm_*$ it remains to show that the induction step Equation~\eqref{eq:chain_matrix_update} converges with now fixed bootstrap target $\thetakm_*$.
This update converges to $\thetak_*$ by Proposition~\ref{prop:richardson_convergence} even for unstable $\tdA$ because $\tdX$ is positive definite. Hence sequential estimation is convergent. In Figure~\ref{fig:expected_update} (left) we estimate a sequence of value functions with their expected update for $T=250$ steps each and can observe convergence to the off-policy target value. For sequential estimation we use a strictly optional hot-start heuristic where after each $250$ update steps we initialize the next $\thetatkp$ with the previous solution $\thetak$ to accelerate convergence.

\begin{proposition}
Expected sequential estimation of Chained TD is convergent.
\end{proposition}
\begin{proof}
Iterating Eq.~\eqref{eq:chain_matrix_update} converges  due to Proposition ~\ref{prop:richardson_convergence} for sufficiently small $\alpha$ because $\tdX$ is positive definite. 
\end{proof}

\paragraph{Concurrent Estimation}
We call \emph{Concurrent Estimation} the process where all value functions in the chain are updated simultaneously at each time step. In contrast to sequential training we do not assume that the previous value function in the chain has converged. This estimation may for example be more convenient for online learning, but requires a new proof of convergence. The proof works as follows: We will show that the matrix $\tdM$ (see \eqref{eq:chain_joint_matrix}) -- corresponding to the joint TD update of all parameters -- has solely positive eigenvalues. Then viewing 
this expected concurrent update (see \eqref{eq:chain_joint_update})
as Richardson Iteration implies the existence of a unique solution and convergence for a suitable step-size $\alpha$.

In Figure~\ref{fig:expected_update} (center) we train a sequence of value functions with their expected concurrent update and observe convergence in accordance with the proposition below. We also observe oscillations in the value predictions in early training.
This effect vanishes eventually as the parameters converge.

Nevertheless such oscillations may be inconvenient and their mitigation provides an interesting direction for future research. We present a simple mitigation technique of gradient normalization to reduce the pre-convergence oscillation magnitude in Figure~\ref{fig:expected_update}~(right).

\subsubsection{The Expected Concurrent Update of Chained TD }
The expected update of all chain parameters $\{\thetak\}_{k \in \setNaturalToK}$ can be written as a joint update in matrix form using one block structured update matrix $\tdM$. 
\newcommand{\tdbjoint}{\vec{b^\dagger}}
\begin{equation}
    \underbrace{\tp{\tdthetas}}_{\boldsymbol{\thetatp}} = \t{\tdthetas}
    + \alpha \left( \underbrace{\tdbs}_{\tdbjoint} - \tdM
    \underbrace{\t{\tdthetas}}_{\boldsymbol{\thetat}} \right) 
\label{eq:chain_joint_update}
\end{equation}
with
\begin{equation}
\tdM \defeq \tdMs
\label{eq:chain_joint_matrix}
\end{equation}

\subsubsection{Fixed Point and Convergence of the Concurrent TD Update}
The formulation above allows us employ Richardson Iteration to analyze the convergence properties of all simultaneously changing parameters by investigating $\tdM$. As we will see $\tdM$ has only positive eigenvalues such that convergence to the unique solution $\boldsymbol{\theta_*} = \tdM^{-1} \tdbjoint$ follows. Contrary to GTD2 and TDC a single step-size suffices.

\begin{proposition}
$\tdM$ has only positive eigenvalues.
\label{prop:chain_joint_positive}
\end{proposition}
\begin{proof}
We make use of the fact that the eigenvalues of a triangular block matrix are the union of eigenvalues of the diagonal blocks. The diagonal blocks are $\tdAmu$ and $\tdX$. Since both are positive definite $\tdM$ has positive eigenvalues.
\end{proof}

\begin{proposition}
The expected concurrent update has the same unique fixed point as the sequential update: $\boldsymbol{\theta_*} = [\theta_\mu, \theta^1_*, \cdots, \theta^K_*]$.
\label{prop:chain_same_fixedpoint}
\end{proposition}
\begin{proof}
From Proposition~\ref{prop:chain_joint_positive} it follows that $\tdM$ is invertible hence
$\boldsymbol{\theta_*} = \tdM^{-1} \tdbjoint$
is the unique fixed point of the joint update.
Block-wise solving $\tdM^{-1} \tdbjoint$ leads to an identical recursion as Equation~\eqref{eq:chain_fixpoint_recursion} -- the sequential fixed points. 
\end{proof}

\begin{proposition}
Expected concurrent chained TD is convergent. The expected update converges to the fixed  point $\boldsymbol{\theta_*} = [\theta_\mu, \theta^1_*, \cdots, \theta^K_*]$ given a suitably small step-size.
\end{proposition}
\begin{proof}
Convergence to $\boldsymbol{\theta_*} = \tdM^{-1} \tdbjoint$ follows from Proposition~\ref{prop:chain_joint_positive} (key matrix $\tdM$ has positive eigenvalues) and Proposition~\ref{prop:richardson_convergence} (positive eigenvalues imply convergence). Then $\boldsymbol{\theta_*} = [\theta_\mu, \theta^1_*, \cdots, \theta^K_*]$ by Proposition~\ref{prop:chain_same_fixedpoint}.
\end{proof}

\section{Empirical Study}
\label{sec:experiments}

\begin{table*}[t]
\centering
\begin{tabular}{ c | c c c | c c c }
 RMSE for MDP & Baird & Baird-Reward & Threestate & Baird & Baird-Reward & Threestate  \\
 with discount  & \multicolumn{3}{c |}{$\gamma=0.9$} & \multicolumn{3}{c}{$\gamma=0.99$}  \\ 
 with reward    & No & Yes & Yes & No & Yes & Yes \\
 \toprule
TD (no correction)                &     0.0 &           10.0 &       10.1 &   0.0 &         99.3 &         102.8 \\
Off-Policy TD    &     div &            div &        div &   div &          div &           div \\
\midrule
ETD                &     0.0 &            div &        0.0 & 136.7 &        div &           div \\
GTD2               &     0.2 &            0.1 &        0.0 &  12.5 &         83.4 &         139.6 \\
TDC                &     0.3 &            0.3 &        0.0 &  13.3 &         87.0 &          43.6 \\
\midrule
Concurrent Chained TD            &     0.0 &            0.4 &        0.1 &   0.0 &         72.6 &          77.9 \\
Sequential Chained TD &     0.0 &            0.0 &        0.0 &   0.0 &          0.0 &           0.2 \\
\bottomrule
\end{tabular}
\caption{\label{tab:best_1_step} 
Evaluation of various 1-step TD algorithms on several MDPs. Observe that MDPs with large discount and rewards (Baird-Reward and Threestate) are the most challenging and that only sequentially chained TD learning obtains RMSE close to 0. Results with RMSE larger than $150$ are considered divergent.
}
\end{table*}

In the previous sections we have shown that the expected update of chained TD is guaranteed to converge for sequential and concurrent parameter updates. Furthermore we have shown that it is unbiased wrt. $\thetapi$ under mild assumptions. In this section we empirically study how the corresponding stochastic update for chained TD converges on a selection of MDPs and observe favourable results. 

We compare to regular off-policy TD and Emphatic Temporal Differences (ETD), and two forms of Gradient Temporal Difference Learning (GTD2 and TDC). All but the foremost are proven to be stable and have different trade-offs in practice. 
In our study we observe that ETD, GTD2 and TDC can suffer more from variance - and may even diverge for that reason - than chained TD if the discount is large $\gamma=0.99$. However they converge faster if the discount is small $\gamma=0.9$.

\subsection{Methodology}
While our method could also be applied offline, here we consider online off-policy learning where the stochastic update samples one transition at at time according to $\mu$ and then updates all parameters using temporal difference learning to estimate $v_\pi$. For chained-TD we bootstrap from the previous value function in the chain, while the first chain estimates $v_\mu$ with TD$(0)$. 

We consider three MDPs all with small discount of $\gamma=0.9$ and large discount $\gamma=0.99$ and evaluate algorithms according to the following experimental protocol: We evaluate the product of all relevant hyper-parameters for $100,000$ transitions and select the result with the lowest mean squared error averaged over the final $50\%$ of transitions and over 10 seeds. We then select the best hyper-parameters and rerun the experiment with 100 new seeds. As hyper-parameters we consider all step-sizes $\alpha$ form the range $S=\{2^{-i/3} | i \in \{1, \dots, 40 \}\}$ (i.e. logarithmically spaced between $9.6\times10^{-5}$ and $0.5$), for GTD2 and TDC we also consider all secondary step-sizes $\beta$ form the same range, for chained TD we consider chains of length $256$ and evaluate the performance of only 9 indices $k \in I=\{2^{i} | i \in \{0, \dots, 8 \}\}$.  This can be seen as a more efficient concurrent equivalent of experimenting with 9 different chain length separately. For sequential chained TD we split the training into windows of $T \in\{25, 50, 100, 200\}$ steps during which only one $\thetak$ is estimated and all others kept unchanged. To prevent pollution from accidentally good initial values we initialize all parameters from a Gaussian distribution with $\sigma=100$ such that errors at $t=0$ are high.

\subsection{Diagnostic Markov Decision Processes}
\subsubsection{Baird's MDP With and Without Rewards}
Baird's MDP is a classic example that demonstrates the divergence of off-policy TD with linear function approximation and has been used to evaluate the convergence of novel approaches. Originally proposed with a discount of $\gamma=0.99$ it is often used with $\gamma=0.9$, which results in lower variance updates. We consider both discounts. Furthermore we introduce a version of Baird's MDP with rewards as the rewards of the classic MDP are all 0. By introducing rewards we are able to investigate the bias of various convergent algorithms. To see why this interesting consider divergent off-policy TD with a large l2 regularization on $\theta$. If the regularization is large enough it will push all parameters to 0, hence the value prediction will be 0 and match the target value of 0. This would be a stable but biased prediction if $v_\pi \neq 0$. To measure the bias we introduce rewards such that $v_\pi=\frac{1}{1-\gamma}$ (i.e $10$ or $100$) and $v_{\mu}=0$ by rewarding each "solid" action with $1$ and each "dashed" action with $-\frac{1}{6}$. We refer to this MDP as the \emph{Baird-Reward MDP}.

\subsubsection{The Threestate MDP}
Inspired by the Twostate MDP~\cite{Tsitsiklis:1997,Sutton:2016} that demonstrates the divergence of off-policy TD concisely without rewards and with only two states, we propose the \emph{Threestate MDP} with one middle state and two border states and two actions: "left" with $-1$ reward and "right" with $1$ reward, leading to the corresponding neighbouring states or remaining if there is no further state in that direction. The starting state distribution is uniform. As with Baird-Reward introducing rewards permits us to measure the bias and convergence speed of various off-policy value predictors. We define
$\Phi = 
    \begin{bmatrix}
    1       & 1 & 1 \\
    1       & 2 & 1 \\
    2       & 2 & 1
\end{bmatrix}$
with full rank such that any state-value combination can be represented by a linear function. Hence any observed bias is entirely due to the evaluated algorithm. The target policy is "right" at all states while the behaviour is uniform. Again we consider $\gamma=0.9$ and $\gamma=0.99$ and observe that $v_\pi=\frac{1}{1-\gamma}$ and $v_{\mu}=0$. 

\begin{figure}[t]
\centering
\includegraphics[width=0.45\textwidth]{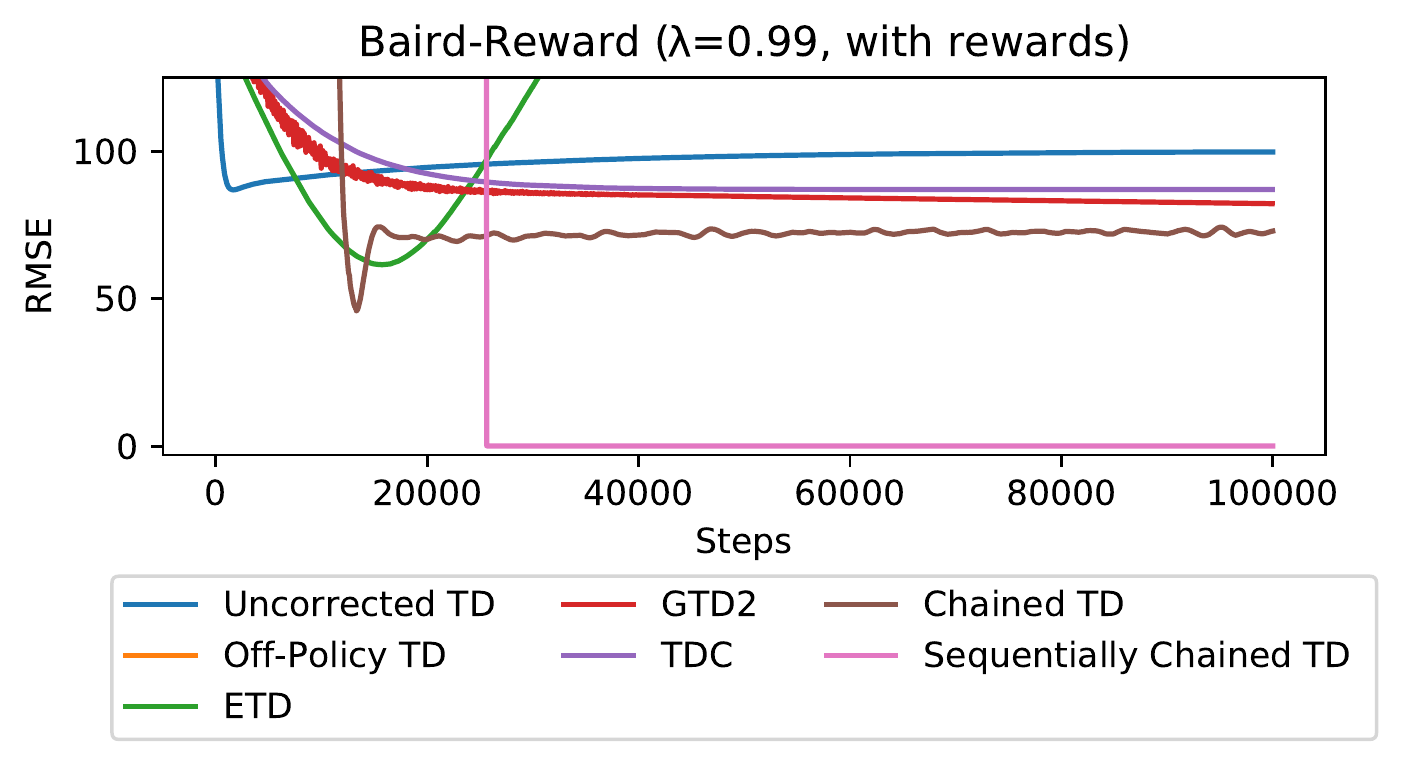}
\caption{Learning process of the 1-step TD algorithms corresponding to Table~\ref{tab:best_1_step} on Baird's MDP with rewards. Observe that chained TD learning reduces the RMSE most with only sequential chained TD learning reducing the error entirely. Off-policy TD diverged and is off the scale.}
\label{fig:comparison_best_selected}
\end{figure}

\subsection{Experimental Results}
\subsubsection{Insights into 1-Step TD Estimators}
In Table~\ref{tab:best_1_step} we evaluate popular TD off-policy value estimators on three MDPs each with two discounts ($\gamma=0.9$ and $\gamma=0.99$) and can observe that the larger discount is more challenging: Only sequential chained TD obtains an RMSE close to $0$ on all MDPs and discounts. 

Furthermore we provide learning curves for Baird-Reward with discount $\gamma=0.99$ in Figure~\ref{fig:comparison_best_selected}. Learning curves corresponding to all entries in the table can be found in the appendix.

At first we note that naive TD estimation (without off-policy correction) of $v_\mu$ is stable but its bias wrt. $v_\pi$ is noticeable in MDPs with rewards (Baird-Reward and Threestate). It is desirable that an off-policy estimator is at least better than this naive baseline. However on Bairds MDP without rewards it inadvertently predicts the correct value, hence we invite the reader to focus on Baird-Reward and Threestate. Next we observe that off-policy TD indeed either diverges or obtains a large error where divergence could be slowed down by a low learning rate.

ETD, GTD2 and TDC mostly fare well where the discount is small $\gamma=0.9$. For $\gamma=0.99$ ETD diverges on the MDPs with rewards. GTD2 and TDC obtain errors on Threestate of $139.6$ and $43.6$ respectively, on Baird-Reward they reduce the RMSE to $83.4$ and $87.0$. 

Concurrent chained TD converges to the true value for small discounts $\gamma=0.9$ and Baird irrespective of discount, while for large discount reducing the error to $72.6$ and $77.9$ on the challenging Baird-Reward and Threestate MDPs.
Finally we observe that sequential chained TD converges close to the true value for all considered MDPs and discounts.

\subsubsection{Chained N-step Estimators}
\begin{figure}[t]
\centering
\includegraphics[width=0.45\textwidth]{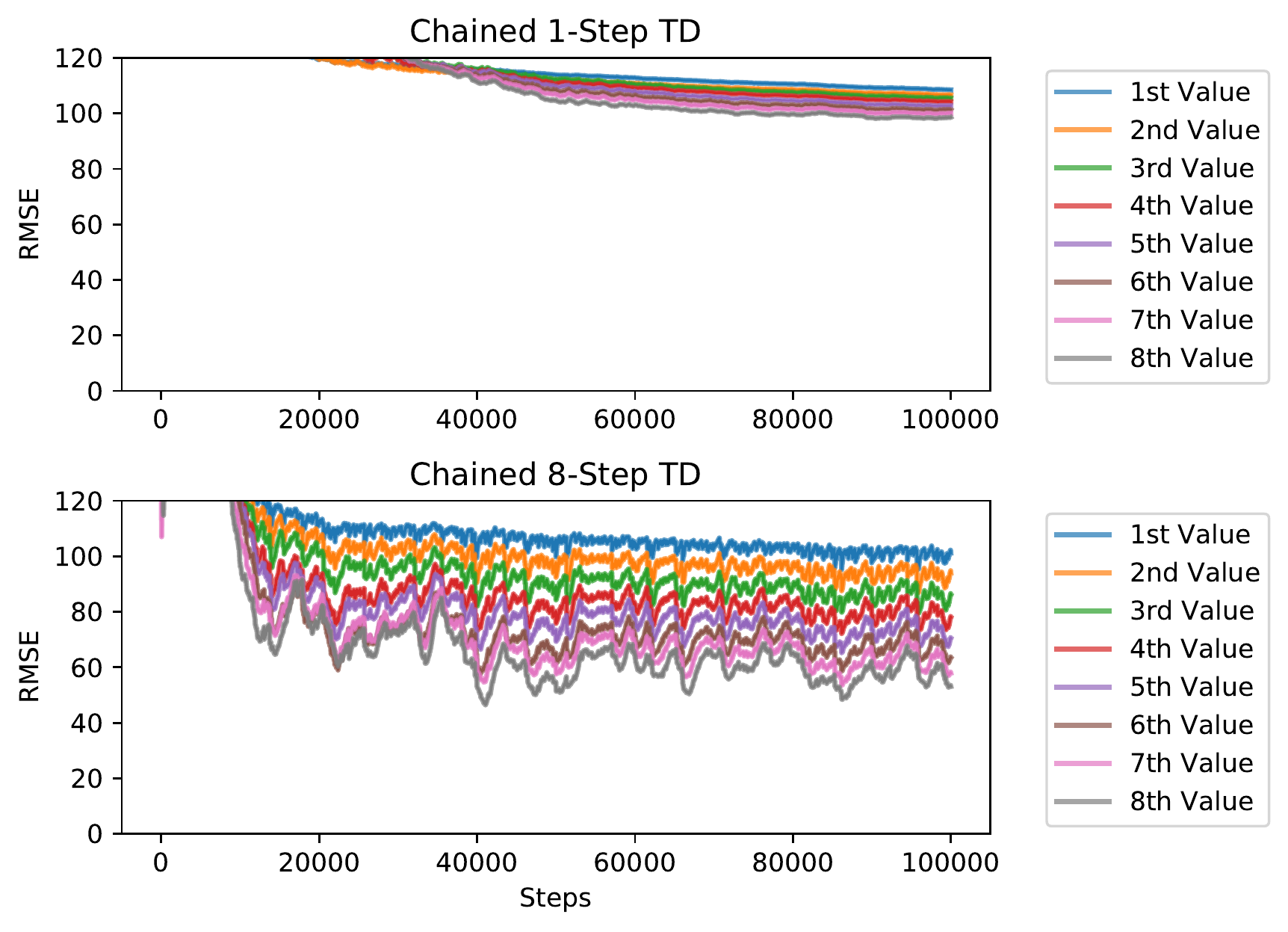}
\caption{Convergence behaviour with increasing $k$ for chains with chained 1-step (\textbf{top}) vs. chained 8-step off-policy TD (\textbf{bottom}). We present the RMSE of first eight \kth-values that are learned concurrently i.e. each bootstrapping off the previous value prediction. Observe how this leads to a sequence of increasingly better predictions.
Finally note that the RMSE of the \numberth{8} 8-step value prediction is lower on Threestate than the concurrent chained TD presented in Table~\ref{tab:best_1_step} which only contains 1-step algorithms.
}
\label{fig:n_step_chains}
\end{figure}

The principle of chaining value functions can also be applied to n-step estimators. N-step estimators predict the value of taking $n$ steps with target policy and then following the policy corresponding to the bootstrap target. Chaining $k$ such estimators results in a total prediction of $m = k \times n$ steps following $\pi$. This allows to predict the $m$-step expedition value $v^m$ with a chain of fewer value functions.

In Figure~\ref{fig:n_step_chains} we confirm this fact empirically on the Threestate MDP with $\gamma=0.99$. One can see that a ($m=8$)-step chain of length $k=8$ attains a much lower RMSE than a ($m=1$)-step chain of the same length.
This suggests that n-step estimators may permit the use of shorter chains.
Using importance sampling estimators to reduce the total length of the chain comes at the cost of increased variance. On the other hand it may come at the benefits of faster convergence and lower bias.
Overall there is a bias, variance and computational complexity trade-off and $n$-step estimators allow to trade this off through the choice of $n$ and $k$.

\section{Conclusion}
We present a novel family of off-policy value prediction algorithms that is convergent by construction. It works through chaining estimators that themselves do not need to be convergent. In particular we prove convergence of sequential and concurrent chained TD, which comes with the intuitive interpretation of estimating the value of a k-step expedition: following $\pi$ for $k$ steps and then following $\mu$ indefinitely.

Furthermore we provide an analytic formula for the bias of chained TD which can be used to derive three insights: 
\begin{itemize}
    \item Sequential chained TD is equivalent to TD with target networks that are switched slowly (i.e.  once the current objective has converged) allowing us to compute the bias of such target-network TD and note $\spectral{\gamma\tdProj\Ppi}<1$ as the precise condition for its convergence. 
    \item Sequential and concurrent chained TD are always convergent but may be biased, while off-policy TD may diverge and yield unbounded values when computing $\thetapi$.
    \item Chained TD is unbiased wrt. $\thetapi$ in the theoretical limit of using infinitely many value functions if $\spectral{\gamma\tdProj\Ppi}<1$ e.g. on Baird's MDP where off-policy TD diverges.
\end{itemize}

Future work may be directed to investigate chaining other updates e.g. chained V-trace \citep{IMPALA}, chained Expected SARSA \citep{vanSeijen:2009}, chained Retrace \citep{Munos:2016Retrace} and to investigate the bias vs. variance trade-off of those chained estimators.
For example better multi-step off-policy returns may lead to faster convergence. Chaining importance-sampling-free Q-learning can be used to estimate values off-policy even if no action probabilities were recorded. This may be useful to learn when the behaviour policy is unknown, e.g., from human demonstrations. 
Finally, for concurrent chaining, where all value functions in the chain are learned at the same time, the choice of which to select for acting may be taken at run-time, and potentially learnt, for example via bandits \cite{Badia:2020Never} or meta-gradients \cite{Sutton:1992IDBD,Xu:2018}.

\section*{Acknowledgements} 
We would like to thank Tom Zahavy and the anonymous AAAI 2022 reviewers for their valuable feedback.

\bibliography{aaai22}

\begin{thebibliography}{27}
\providecommand{\natexlab}[1]{#1}

\bibitem[{Badia et~al.(2020)Badia, Sprechmann, Vitvitskyi, Guo, Piot,
  Kapturowski, Tieleman, Arjovsky, Pritzel, Bolt, and
  Blundell}]{Badia:2020Never}
Badia, A.~P.; Sprechmann, P.; Vitvitskyi, A.; Guo, D.; Piot, B.; Kapturowski,
  S.; Tieleman, O.; Arjovsky, M.; Pritzel, A.; Bolt, A.; and Blundell, C. 2020.
\newblock Never Give Up: Learning Directed Exploration Strategies.
\newblock In \emph{International Conference on Learning Representations}.

\bibitem[{Baird(1995)}]{Baird:1995}
Baird, L. 1995.
\newblock Residual Algorithms: Reinforcement learning with function
  approximation.
\newblock In \emph{Machine Learning: Proceedings of the Twelfth International
  Conference}, 30--37.

\bibitem[{Brandfonbrener et~al.(2021)Brandfonbrener, Whitney, Ranganath, and
  Bruna}]{brandfonbrener2021offline}
Brandfonbrener, D.; Whitney, W.~F.; Ranganath, R.; and Bruna, J. 2021.
\newblock Offline {RL} Without Off-Policy Evaluation.
\newblock In \emph{Thirty-Fifth Conference on Neural Information Processing
  Systems}.

\bibitem[{De~Asis et~al.(2020)De~Asis, Chan, Pitis, Sutton, and
  Graves}]{DeAsis:2020FixedHorizon}
De~Asis, K.; Chan, A.; Pitis, S.; Sutton, R.; and Graves, D. 2020.
\newblock Fixed-Horizon Temporal Difference Methods for Stable Reinforcement
  Learning.
\newblock \emph{Proceedings of the AAAI Conference on Artificial Intelligence},
  34(04): 3741--3748.

\bibitem[{Espeholt et~al.(2018)Espeholt, Soyer, Munos, Simonyan, Mnih, Ward,
  Doron, Firoiu, Harley, Dunning, Legg, and Kavukcuoglu}]{IMPALA}
Espeholt, L.; Soyer, H.; Munos, R.; Simonyan, K.; Mnih, V.; Ward, T.; Doron,
  Y.; Firoiu, V.; Harley, T.; Dunning, I.; Legg, S.; and Kavukcuoglu, K. 2018.
\newblock IMPALA: Scalable Distributed Deep-RL with Importance Weighted
  Actor-Learner Architectures.
\newblock In \emph{Arxiv}.

\bibitem[{Gulcehre et~al.(2021)Gulcehre, Colmenarejo, ziyu wang, Sygnowski,
  Paine, Zolna, Chen, Hoffman, Pascanu, and
  de~Freitas}]{gulcehre2021addressing}
Gulcehre, C.; Colmenarejo, S.~G.; ziyu wang; Sygnowski, J.; Paine, T.; Zolna,
  K.; Chen, Y.; Hoffman, M.; Pascanu, R.; and de~Freitas, N. 2021.
\newblock Addressing Extrapolation Error in Deep Offline Reinforcement
  Learning.

\bibitem[{Hackman(2013)}]{Hackman:2013}
Hackman, L.~M. 2013.
\newblock \emph{Faster Gradient-TD Algorithms}.
\newblock Master's thesis, University of Alberta.

\bibitem[{Hauskrecht and Fraser(2000)}]{Hausknecht:2000PlanningTreatment}
Hauskrecht, M.; and Fraser, H. 2000.
\newblock Planning treatment of ischemic heart disease with partially
  observable Markov decision processes.
\newblock \emph{Artificial Intelligence in Medicine}, 18(3): 221--244.

\bibitem[{Hester et~al.(2018)Hester, Vecerik, Pietquin, Lanctot, Schaul, Piot,
  Horgan, Quan, Sendonaris, Osband, Dulac-Arnold, Agapiou, Leibo, and
  Gruslys}]{Hester:2018DeepDemonstrations}
Hester, T.; Vecerik, M.; Pietquin, O.; Lanctot, M.; Schaul, T.; Piot, B.;
  Horgan, D.; Quan, J.; Sendonaris, A.; Osband, I.; Dulac-Arnold, G.; Agapiou,
  J.; Leibo, J.; and Gruslys, A. 2018.
\newblock Deep Q-learning From Demonstrations.

\bibitem[{Jaderberg et~al.(2017)Jaderberg, Mnih, Czarnecki, Schaul, Leibo,
  Silver, and Kavukcuoglu}]{jaderberg:2016reinforcement}
Jaderberg, M.; Mnih, V.; Czarnecki, W.~M.; Schaul, T.; Leibo, J.~Z.; Silver,
  D.; and Kavukcuoglu, K. 2017.
\newblock Reinforcement learning with unsupervised auxiliary tasks.
\newblock \emph{International Conference on Learning Representations}.

\bibitem[{Maei(2011)}]{Maei:2011}
Maei, H.~R. 2011.
\newblock \emph{Gradient temporal-difference learning algorithms}.
\newblock Ph.D. thesis, University of Alberta.

\bibitem[{Mazoure et~al.(2021)Mazoure, Mineiro, Srinath, Sedeh, Precup, and
  Swaminathan}]{mazoure2021improving}
Mazoure, B.; Mineiro, P.; Srinath, P.; Sedeh, R.~S.; Precup, D.; and
  Swaminathan, A. 2021.
\newblock Improving Long-Term Metrics in Recommendation Systems using
  Short-Horizon Reinforcement Learning.
\newblock arXiv:2106.00589.

\bibitem[{Mnih et~al.(2015)Mnih, Kavukcuoglu, Silver, Rusu, Veness, Bellemare,
  Graves, Riedmiller, Fidjeland, Ostrovski, Petersen, Beattie, Sadik,
  Antonoglou, King, Kumaran, Wierstra, Legg, and Hassabis}]{Mnih:2015}
Mnih, V.; Kavukcuoglu, K.; Silver, D.; Rusu, A.~A.; Veness, J.; Bellemare,
  M.~G.; Graves, A.; Riedmiller, M.; Fidjeland, A.~K.; Ostrovski, G.; Petersen,
  S.; Beattie, C.; Sadik, A.; Antonoglou, I.; King, H.; Kumaran, D.; Wierstra,
  D.; Legg, S.; and Hassabis, D. 2015.
\newblock Human-level control through deep reinforcement learning.
\newblock \emph{Nature}, 518(7540): 529--533.

\bibitem[{Munos et~al.(2016)Munos, Stepleton, Harutyunyan, and
  Bellemare}]{Munos:2016Retrace}
Munos, R.; Stepleton, T.; Harutyunyan, A.; and Bellemare, M. 2016.
\newblock Safe and Efficient Off-Policy Reinforcement Learning.
\newblock In Lee, D.; Sugiyama, M.; Luxburg, U.; Guyon, I.; and Garnett, R.,
  eds., \emph{Advances in Neural Information Processing Systems}, volume~29.
  Curran Associates, Inc.

\bibitem[{Richardson(1911)}]{Richardson:1911}
Richardson, L.~F. 1911.
\newblock The Approximate Arithmetical Solution by Finite Differences of
  Physical Problems Involving Differential Equations, with an Application to
  the Stresses in a Masonry Dam.
\newblock \emph{Philosophical Transactions of the Royal Society of London.
  Series A, Containing Papers of a Mathematical or Physical Character}, 210:
  307--357.

\bibitem[{Sutton(1992)}]{Sutton:1992IDBD}
Sutton, R.~S. 1992.
\newblock Adapting bias by gradient descent: An incremental version of
  delta-bar-delta.
\newblock In \emph{Proceedings of the Tenth National Conference on Artificial
  Intelligence}, 171--176. MIT Press.

\bibitem[{Sutton and Barto(2018)}]{SuttonBarto:2018}
Sutton, R.~S.; and Barto, A.~G. 2018.
\newblock \emph{Reinforcement Learning: An Introduction}.
\newblock The MIT press, Cambridge MA.

\bibitem[{Sutton et~al.(2009)Sutton, Maei, Precup, Bhatnagar, Silver,
  Szepesv{\'a}ri, and Wiewiora}]{Sutton:2009}
Sutton, R.~S.; Maei, H.~R.; Precup, D.; Bhatnagar, S.; Silver, D.;
  Szepesv{\'a}ri, C.; and Wiewiora, E. 2009.
\newblock Fast gradient-descent methods for temporal-difference learning with
  linear function approximation.
\newblock In \emph{Proceedings of the 26th Annual International Conference on
  Machine Learning (ICML 2009)}, 993--1000. ACM.

\bibitem[{Sutton, Mahmood, and White(2016)}]{Sutton:2016}
Sutton, R.~S.; Mahmood, A.~R.; and White, M. 2016.
\newblock An Emphatic Approach to the Problem of Off-policy Temporal-Difference
  Learning.
\newblock \emph{Journal of Machine Learning Research}, 17(73): 1--29.

\bibitem[{Sutton et~al.(2011)Sutton, Modayil, Delp, Degris, Pilarski, White,
  and Precup}]{Sutton:2011}
Sutton, R.~S.; Modayil, J.; Delp, M.; Degris, T.; Pilarski, P.~M.; White, A.;
  and Precup, D. 2011.
\newblock Horde: A scalable real-time architecture for learning knowledge from
  unsupervised sensorimotor interaction.
\newblock In \emph{The 10th International Conference on Autonomous Agents and
  Multiagent Systems-Volume 2}, 761--768. International Foundation for
  Autonomous Agents and Multiagent Systems.

\bibitem[{Tsitsiklis and {Van Roy}(1997)}]{Tsitsiklis:1997}
Tsitsiklis, J.~N.; and {Van Roy}, B. 1997.
\newblock An analysis of temporal-difference learning with function
  approximation.
\newblock \emph{{IEEE} Transactions on Automatic Control}, 42(5): 674--690.

\bibitem[{van Hasselt et~al.(2018)van Hasselt, Doron, Strub, Hessel, Sonnerat,
  and Modayil}]{Hasselt:2018Deadly}
van Hasselt, H.; Doron, Y.; Strub, F.; Hessel, M.; Sonnerat, N.; and Modayil,
  J. 2018.
\newblock Deep Reinforcement Learning and the Deadly Triad.
\newblock \emph{CoRR}, abs/1812.02648.

\bibitem[{{van Hasselt}, Mahmood, and Sutton(2014)}]{vanHasselt:2014}
{van Hasselt}, H.; Mahmood, A.~R.; and Sutton, R.~S. 2014.
\newblock Off-policy {TD}($\lambda$) with a true online equivalence.
\newblock In \emph{Uncertainty in Artificial Intelligence}.

\bibitem[{{van Seijen} et~al.(2009){van Seijen}, {van Hasselt}, Whiteson, and
  Wiering}]{vanSeijen:2009}
{van Seijen}, H.; {van Hasselt}, H.; Whiteson, S.; and Wiering, M. 2009.
\newblock A theoretical and empirical analysis of {Expected Sarsa}.
\newblock In \emph{Proceedings of IEEE Symposium on Adaptive Dynamic
  Programming and Reinforcement Learning}, 177--184.

\bibitem[{Wiering and van Hasselt(2007)}]{wiering2007qv}
Wiering, M.~A.; and van Hasselt, H. 2007.
\newblock Two Novel On-policy Reinforcement Learning Algorithms based on
  TD($\lambda$)-methods.
\newblock In \emph{2007 IEEE International Symposium on Approximate Dynamic
  Programming and Reinforcement Learning}, 280--287.

\bibitem[{Wiering and van Hasselt(2009)}]{wiering2009qvfamily}
Wiering, M.~A.; and van Hasselt, H. 2009.
\newblock The QV family compared to other reinforcement learning algorithms.
\newblock In \emph{2009 IEEE Symposium on Adaptive Dynamic Programming and
  Reinforcement Learning}, 101--108.

\bibitem[{Xu, van Hasselt, and Silver(2018)}]{Xu:2018}
Xu, Z.; van Hasselt, H.~P.; and Silver, D. 2018.
\newblock Meta-Gradient Reinforcement Learning.
\newblock In Bengio, S.; Wallach, H.; Larochelle, H.; Grauman, K.;
  Cesa-Bianchi, N.; and Garnett, R., eds., \emph{Advances in Neural Information
  Processing Systems}, volume~31. Curran Associates, Inc.

\end{thebibliography}

\appendix
\section*{Appendix}

\begin{algorithm}[tb]
\caption{\textbf{Sequential Chained TD} with optional hot-start heuristic.}
\label{alg:algorithm_appendix}
\textbf{Input}: $\pi$, $\mu$, number of chains $K$, number of update steps $T$\\
\textbf{Parameter}: step size $\alpha$
\begin{algorithmic}[1]
\State Initialize all $\{\thetak\}_{k\in \setNaturalToK}$ randomly, $t\gets 0$.
\For{$k \gets 0$ to $K$}
\For{$i \gets 1$ to $T$}
\State $t \gets t + 1$
\State Play one action $A_t$ with $\mu$.
\State Observe next state $\stp$ and reward $\tp{R}$.
\If {$k=0$}
\State $\delta \gets \tp{R} + \gamma \tk{v}(\stp) - \tk{v}(\st)$; $\rho \gets 1$
\Else 
\State $\delta \gets \tp{R} + \gamma \tkm{v}(\stp) - \tk{v}(\st)$
\State $\rho \gets \frac{\pi(A_t| \st)}{\mu(A_t| \st)}$
\EndIf
\State $\thetak \gets \thetak + \alpha \rho \delta \nabla_{\theta} \k{v}(\st)$
\EndFor
\State $\thetakp \gets \thetak$ \Comment{Optional hot-start heuristic.}
\EndFor
\State \textbf{return} $\{\thetak\}_{k\in \setNaturalToK}$
\end{algorithmic}
\end{algorithm}

\begin{algorithm}[tb]
\caption{\textbf{Off-Policy Target Network TD}}
\label{alg:target_net}
\textbf{Input}: $\pi$, $\mu$, number of switches $K$, updates per network $T$\\
\textbf{Parameter}: step size $\alpha$
\begin{algorithmic}[1]
\State Initialize all $\{\thetak\}_{k\in \setNaturalToK}$ randomly, $t\gets 0$.
\For{$k \gets 1$ to $K+1$}
\For{$i \gets 1$ to $T$}
\State $t \gets t + 1$
\State Play one action $A_t$ with $\mu$.
\State Observe next state $\stp$ and reward $\tp{R}$.
\State $\delta \gets \tp{R} + \gamma \tkm{v}(\stp) - \tk{v}(\st)$
\State $\rho \gets \frac{\pi(A_t| \st)}{\mu(A_t| \st)}$
\State $\thetak \gets \thetak + \alpha \rho \delta \nabla_{\theta} \k{v}(\st)$
\EndFor
\State Forget parameters $\theta^{k-1}$.
\State $\thetakp \gets \thetak$
\EndFor
\State \textbf{return} $\theta^{K+1}$
\end{algorithmic}
\end{algorithm}

\section{Relation to Target Networks}
Sequential chained TD has a noteworthy connection to off-policy TD learning with target networks that allows us to obtain insights into the later -- when the target networks are switched slowly. By \emph{slowly} we mean that the previous parameters have sufficiently converged.

\subsubsection{Sequential TD is Similar to Target Network TD} When updating its parameters $\thetak$ target network TD bootstraps from returns corresponding an earlier copy of the parameters $\thetakm$ and switches networks i.e. increases $k$ every $T$ steps (see Algorithm~\ref{alg:target_net}).
Recall that chained TD estimates each $\thetak$ only after the previous $\thetakm$ has been estimated --  i.e. increases $k$ every $T$ steps for some large enough $T$ (see Algorithm~\ref{alg:algorithm_appendix}). 
Overall they perform at total of $TK$ steps with the identical update 
$\thetak \gets \thetak + \alpha \rho \delta \nabla_{\theta} \k{v}(\st)$, that are used to estimate $K$ different parameters for $T$ update steps each. While sequential TD returns the history of $K$ parameters, target network TD only returns the final parameters.

Sequential chained TD has a special update for $k=0$, where it estimates the behaviour value $v_\mu$. As we show in the paper this step is optional and does not impact the theoretical analysis. Furthermore chained TD has an optional hot-start heuristic, which accelerates convergence but does not change the fixed point of the update. If $T$ is chosen sufficiently large to ensure convergence it can be omitted. If we exclude both optional steps we can conclude that both algorithms have the same behaviour and hence same fixed points. In the paper we analyzed the fixed point as $T\to \infty$ i.e. when networks are switched slowly.

\subsection{Bias of Slow Target Network TD}
We will now use the insights from sequentially chained TD to analyze the special case of TD with slowly switching target networks. In practice
such an instance of TD would only switch parameters after the previous have converged.

As we have seen the fixed points are identical to sequentially chained TD in this case, hence the bias wrt. $\thetapi\defeq\tdA^{-1}\tdb$ is also identical. From the paper we recall
\begin{align}
\thetak_*
&= \tdA^{-1} \tdb + \underbrace{\gamma^k\invXYi{k}\left(\theta^0_* - \tdA^{-1} \tdb  \right)}_{\mathbf{Bias\ wrt. \ \boldsymbol{\thetapi}}} 
\label{eq:theta_bias_appendix}
\end{align}

By setting $\theta^0_*$ to a random value we replace the optional first estimation step of chained TD and fully recover target network TD which initially bootstraps from a random value.

\subsection{Convergence of Slow Target Network TD}
Equation~\eqref{eq:theta_bias_appendix} computes the distance of slow target network TD to $\thetapi$ after $k$ network switches. If $\spectral{\gamma\tdProj\Ppi}<1$ or equivalently $\spectral{\invXgY} < 1$ the distance decreases with $k$ and hence slow target network TD converges to $\thetapi$ as $K \to \infty$.

\section{Inverse Problem Decomposition View}
As a corollary of Equation \eqref{eq:wk_closed_form}
in Proposition 2 of the main paper, the condition $\spectral{\invXgY} < 1$ implies

$$\tdA^{-1} = \lim_{k \to \infty}  \sum^{k}_{i=0}\invXgYi{i}\tdX^{-1}$$
which we used implicitly. Chained TD can be viewed as computing the right hand sum i.e. iteratively solving the inverse problem.

\section{Details on Vanishing Bias for Chained TD}

In the paper we show that we can decompose the off-policy TD inverse problem
$\thetapi \defeq \tdA^{-1} \tdb$
in to a sequence of subproblems that are each solved by a value function with parameters $\thetak$.
The update for each $\thetak$ is convergent. Furthermore we can compute the distance to $\thetapi$ and observe that it can be arbitrarily reduced with sufficiently large $k$ when $\spectral{\gamma\tdProj\Ppi}<1$. Note that this may even be the case where $\tdA$ has negative eigenvalues i.e. when TD diverges such as on Baird's MDP.

In practice we only use a finite number of value functions $K$.
In Section 3.2 of the paper we consider the case of $K \to \infty$ i.e. what happens when using chaining infinitely many value functions. We observed that the prediction $\k{v}$ becomes unbiased wrt. $\thetapi$ under mild conditions i.e. when:
\begin{equation}
    \lim_{k \to \infty} \gamma^k \norm{\invXYi{k}} = 0
    \label{eq:power_to_zero}
\end{equation}
This condition is equivalent to:
\begin{equation}
    \spectral{\invXgY} < 1
\end{equation}
and the more interpretable condition
\begin{equation}
    \spectral{\gamma\tdProj\Ppi}<1
\end{equation}
where $\rho$ is the spectral radius.
In the paper we show that the condition is often satisfied for random MDPs with discount $\gamma=0.99$.

Alternatively this could easily achieved by selecting a sufficiently small discount i.e. $\gamma<1/\spectral{\tdProj\Ppi} $. However we are also interested when this happens irrespective of discount -- i.e. even for large discounts $\gamma<1$. In this section we provide additional insights into when this is the case.

\subsection{Structure and Implications of $\mathbf{\tdProj}$}
\newcommand{\tdSFzero}{\begin{bmatrix}
    \mat{I}_{f\times f} & 0 \\
    0 & \mat{0}_{(S-f)\times (S-f)}
\end{bmatrix}_{S\times S}}
\newcommand{\tildeZero}{\mat{\tilde{0}}}

$\tdProj$ has a special structure providing insights into $\spectral{\tdProj\Ppi}$ -- stated in Equations~\eqref{eq:z_eigenvalues} and~\eqref{eq:ir_condition} -- that we formally prove in section~\ref{sec:proof_equivalent_condition}. In section~\ref{sec:conjecture} we employ them to draw insights into the condition $\spectral{\gamma\tdProj\Ppi}<1$. In short, $\tdProj$ has mostly zero eigenvalues if the number of states is much larger than the number of features in an MDP. Furthermore all non-zero eigenvalues are $1.0$.

More formally, we will prove in Proposition~\ref{prop:matrix_zp_special_eigenvalues} that $\tdProj$ can be diagonalized as follows
with $F \leq S$ being the number of features and states and $f\leq F$ being the number of ones 
on the diagonal
\begin{equation} \label{eq:z_eigenvalues}
 \tdProj = \mat{V} \tdSFzero \mat{V}^{-1}
\end{equation}
Hence if $S\gg F$ the fraction of non-zero eigenvalues $\frac{f}{S} \leq \frac{F}{S}$ diminishes as the number of states $S$ increases. 

\subsubsection{Equivalent Condition}
Lemma~\ref{lemma:change_of_variable_does_not_change_eigenvalues} implies that $\spectral{\tdProj\Ppi}=\spectral{\mat{V}\tildeZero\mat{V}^{-1}\Ppi}\leq1$ is equivalent to $\spectral{\tildeZero\mat{V}^{-1}\Ppi\mat{V}}\leq1$.
Hence the condition $\spectral{\tdProj\Ppi}\leq 1$ is equivalent to
\begin{equation} \label{eq:ir_condition}
    \spectral{ \underbrace{\tdSFzero}_{\tildeZero} \underbrace{\mat{V}^{-1} \Ppi \mat{V} }_{\mat{Z}}} \leq 1
\end{equation}
with $\spectral{\tildeZero} \leq 1$ and $\spectral{\mat{Z}} \leq 1$ as we show in Proposition~\ref{prop:matrix_z_spectral_bound}.

\subsection{Conjecture}
\label{sec:conjecture}
We conjecture that $\spectral{\tdProj\Ppi} > 1$ is unlikely for random MDPs when the number of states $S$ is much larger than the number of features $F$.

\paragraph{Intuition}
In Equation~\eqref{eq:ir_condition} stating the equivalent condition  -- that we formally prove in the next sub-section -- the matrix
$\tildeZero$ has almost entirely 0 entries. Hence it reduces any vector that it is multiplied with by setting its components to $0$ unless said vector is chosen adversarially. In random MDPs $\mat{Z}$ depends on $\Ppi$ and random $\Phi$ and is hence not chosen adverserially. Accidentally encountering an adversarial $\mat{Z}$ under the constraint that 
$\spectral{\mat{Z}} \leq 1$
seems to become increasingly unlikely when $S \gg F$ which is an acceptable assumption in practice.

\subsection{Proof for the Equivalent Condition}
\label{sec:proof_equivalent_condition}

\begin{lemma}
\label{lemma:change_of_variable_does_not_change_eigenvalues}
$\mat{A}$ has the same eigenvalues as $\mat{B}\mat{A}\mat{B}^{-1}$ for all square matrices $\mat{A}$,  $\mat{B}$ with same shape and $\mat{B}$ full rank.
\end{lemma}
\begin{proof}
Let $\mat{A}=\mat{V}\mat{\Lambda}\mat{V}^{-1}$ be the eigendecomposition of $\mat{A}$ such that the columns of $\mat{V}$ are the eigenvectors and scaled by $\mat{\Lambda}$. Then
$ \mat{B}\mat{A}\mat{B}^{-1} = \left(\mat{B}\mat{V}\right)\mat{\Lambda}\left(\mat{B}\mat{V}\right)^{-1} $
such that the the columns of $\mat{B}\mat{V}$ are the eigenvectors of $\mat{B}\mat{A}\mat{B}^{-1}$ and scaled by the same eigenvalues $\mat{\Lambda}$. 
\end{proof}

\begin{proposition}
\label{prop:matrix_zp_special_eigenvalues}
The TD projection operator $\tdProj$ as defined in Equation~\eqref{eq:matrix_zp}  can be as decomposed as shown in Equation~\eqref{eq:z_eigenvalues} I.e. $\tdProj$ of shape $S \times S$ has at least $(S-F)$ zero eigenvalues. The remaining $F$ eigenvalues are each either 0 or 1.
\end{proposition}
\begin{proof}
Let $\Dmuroot$ be the element-wise square root of the diagonal matrix $\Dmu$, let $\tdU \defeq \trans{\Phi} \Dmuroot$ and note that $\tdU^{\dagger}=\trans{\tdU} \left( \tdU \trans{\tdU} \right)^{-1}$ is the pseudo-inverse of $\tdU$.
We observe that 
\begin{align*}
\tdProj 
    &\defeq \Phi \left( \trans{\Phi} \Dmu \Phi \right)^{-1}  \trans{\Phi} \Dmu \\
    &= \Dmuinvroot \trans{\tdU} \left( \tdU \trans{\tdU} \right)^{-1} \tdU \Dmuroot \\
    &= \Dmuinvroot \tdU^{\dagger} \tdU \Dmuroot
\end{align*}
next we observe that $\tdU^{\dagger} \tdU$ is an orthogonal projection operator and hence has eigenvalues in $\{0, 1\}$. 
Furthermore by Lemma~\ref{lemma:change_of_variable_does_not_change_eigenvalues} we observe that the multiplication by $\Dmuroot$ and its inverse may change the eigenvectors of a matrix but does not change the eigenvalues. Hence $\tdProj$ has the same eigenvalues as $\tdU^{\dagger} \tdU$ i.e. only 0s and 1s.

Finally we observe that $\tdX = \trans{\Phi} \Dmu \Phi $ has shape $F \times F$ which restricts the rank of $\tdProj$ to $F$. Hence at most $F$ eigenvalues of $\tdProj$ can be non-zero and at least $S-F$ eigenvalues must be $0$.
\end{proof}

\begin{proposition}
\label{prop:matrix_z_spectral_bound}
$\spectral{\mat{Z}} \leq 1$ for $\mat{Z} = \mat{V}^{-1} \Ppi \mat{V}$ for any invertible matrix $\mat{V}$ and any stochastic matrix $\Ppi$.
\end{proposition}
\begin{proof}
$\Ppi$ is stochastic hence $\spectral{\Ppi} \leq 1$ and by Lemma~\ref{lemma:change_of_variable_does_not_change_eigenvalues} also $\spectral{\mat{Z}} \leq 1$.
\end{proof}

\subsection{Example for $\spectral{\invXgY} > 1$}

\newcommand{\DmuTwostate}{\begin{bmatrix}
    0.5 & 0 \\
    0 & 0.5
\end{bmatrix}}
\newcommand{\PpiTwostate}{\begin{bmatrix}
    0 & 1 \\
    0 & 1
\end{bmatrix}}
\newcommand{\PhiTwostate}{\begin{bmatrix}
    1 \\
    2
\end{bmatrix}}
\newcommand{\PhiTransTwostate}{\begin{bmatrix}
    1 & 2
\end{bmatrix}}
Based on our experimental insights we conjecture that the $\spectral{\invXgY} < 1$ condition (equivalent to $\spectral{\gamma\tdProj\Ppi}<1$) is often met but not always as one can construct examples where this is not the case.

Consider the Twostate MDP from \cite{Tsitsiklis:1997,Sutton:2016}. Here chained TD correctly predicts the target value, but the MDP can be modified so that chained TD is not able to predict the target value with arbitrary accuracy. In the same MDP off-policy TD diverges, which can be argued to be a less graceful failure mode than being biased. The MDP has zero rewards, two states with a single feature
$\trans{\Phi}=\PhiTransTwostate$, $\gamma=0.99$ and policies are defined such that 
$$\Dmu=\DmuTwostate, \Ppi=\PpiTwostate$$
In this MDP chained TD is asymptotically unbiased (i.e. $\lim_{k\to\infty}\thetak_* = \theta_\pi$) if the rewards are zero, but not for any reward structure.

We first observe that the spectral condition 
$\spectral{\invXgY} < 1$ is not met for large discounts. From
\begin{align*}
    \tdX &= \trans{\Phi} \Dmu \Phi = \PhiTransTwostate \DmuTwostate \PhiTwostate = 2.5 \\
    \tdY &= \trans{\Phi} \Dmu \Ppi \Phi = \PhiTransTwostate \DmuTwostate \PpiTwostate \PhiTwostate = 3
\end{align*}
we observe that 
$\invXgY = \gamma\frac{3}{2.5}$
and hence
$\spectral{\invXgY} = \gamma\frac{3}{2.5}$
which is larger than $1$ for discounts $\gamma > 5/6$.

We can further investigate the bias $\thetak_*-\thetapi$ at each $k$ using Proposition 2: 
\begin{align}
    \thetak_*-\thetapi
    &=\gamma^k\invXYi{k}\left(\theta^0 - \tdA^{-1} \tdb  \right) \\
    &= \left( \gamma\frac{3}{2.5} \right)^k \left(\theta^0 - \tdA^{-1} \tdb  \right)
\end{align}
For $\gamma > 5/6$ the asymptotic bias ($\lim_{k\to \infty}\thetak_*-\thetapi$) can only be zero if $\theta^0 = \tdA^{-1} \tdb$. For our heuristic where $\theta^0 \defeq \tdAmu^{-1} \tdbmu$ this would only be the case if $\tdA^{-1} \tdb = \tdAmu^{-1} \tdbmu$  -- for example for zero rewards or $\mu=\pi$ but not in general. 
Despite not being asymptotically unbiased in this example each value function of chained TD is guaranteed to converge to the fixed point $\thetak_*$. Hence it biased and convergent for fixed $k$. This is an improvement over regular TD which diverges for this and other MDPs like Baird's counter example. Finally recall that chained TD is both convergent and asymptotically unbiased for Baird's counter example with and without rewards as we showed empirically. There $\tdA$ has negative eigenvalues but $\spectral{\invXgY} < 1$.

\section{Details on Gradient Normalization}
In Figure 2 (center) of the main paper we showed how concurrent estimation oscillates prior to convergence. We also mention that a simple mitigation technique of Gradient Normalization can be used to reduce those oscillations. We only use this normalization for the experiment in Figure 2 (right). Given any expected TD update vector $g$ it transforms it into $g'=\frac{g}{\norm{g}}$ prior to the update.

The presented experiment is intended to motivate further research into such techniques.
A detailed evaluation is out of scope of this paper. All other experiments were run without it.

\section{Learning Curves}
In Figure~\ref{fig:comparison_all_best_selected} we present learning curves corresponding to Table 1 in the main paper.
\begin{figure*}[t]
\centering
\includegraphics[width=0.8\textwidth]{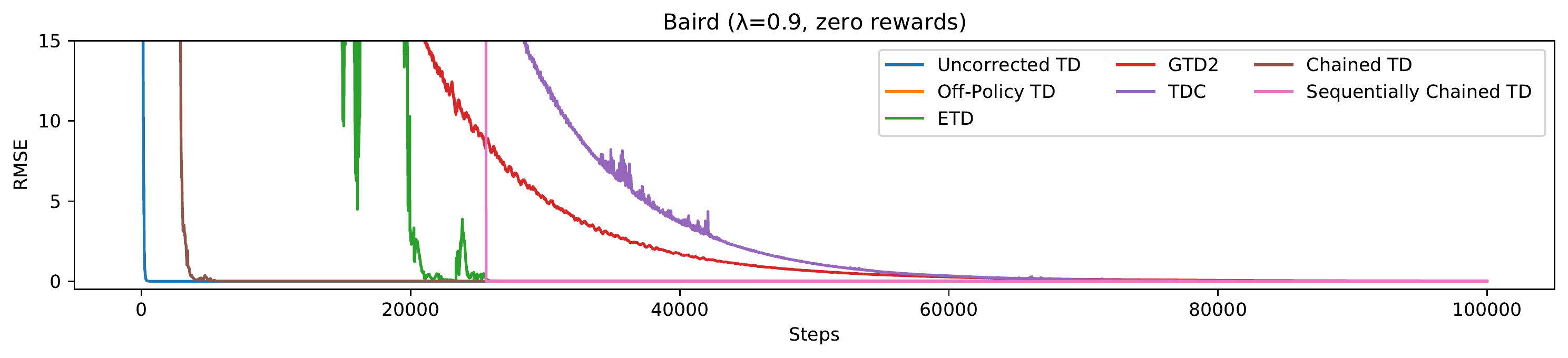}
\includegraphics[width=0.8\textwidth]{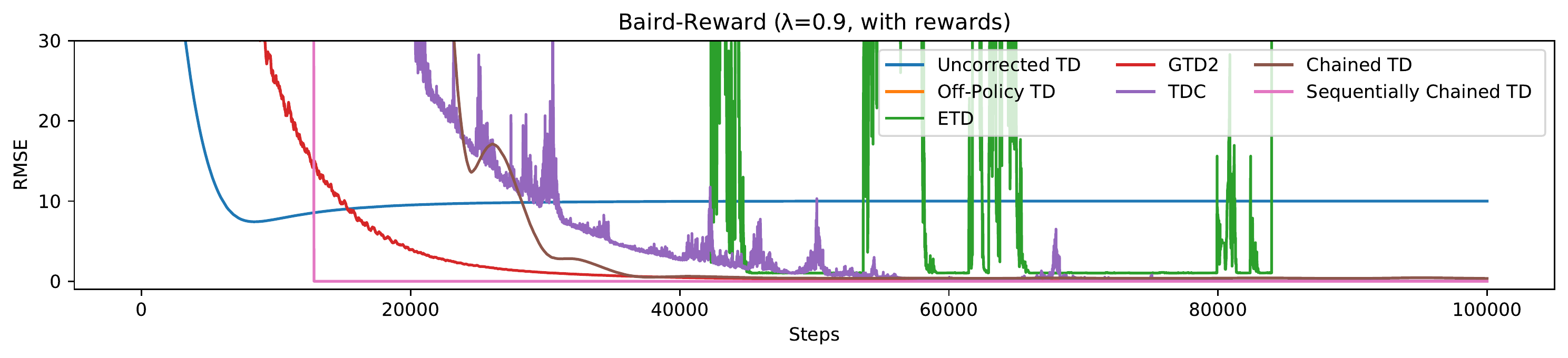}
\includegraphics[width=0.8\textwidth]{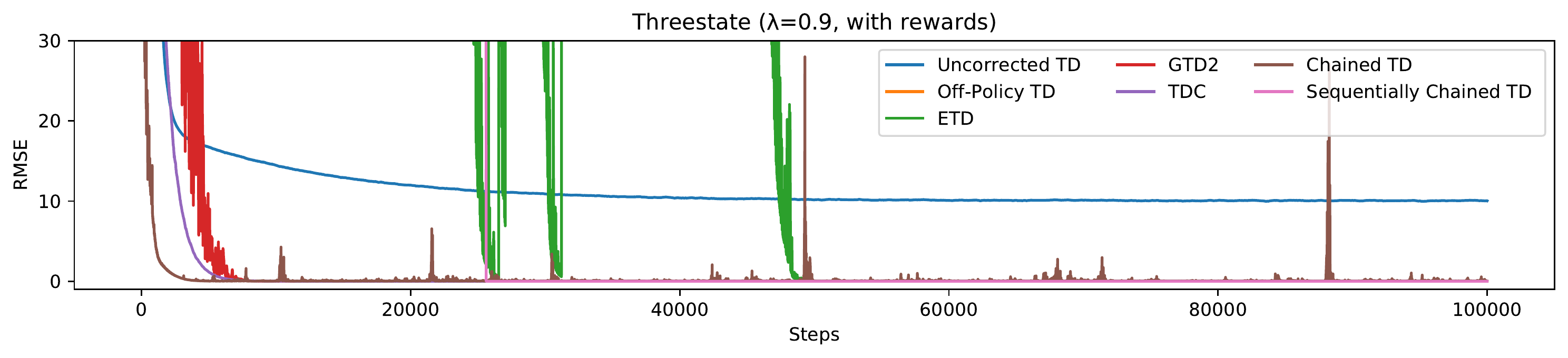}
\includegraphics[width=0.8\textwidth]{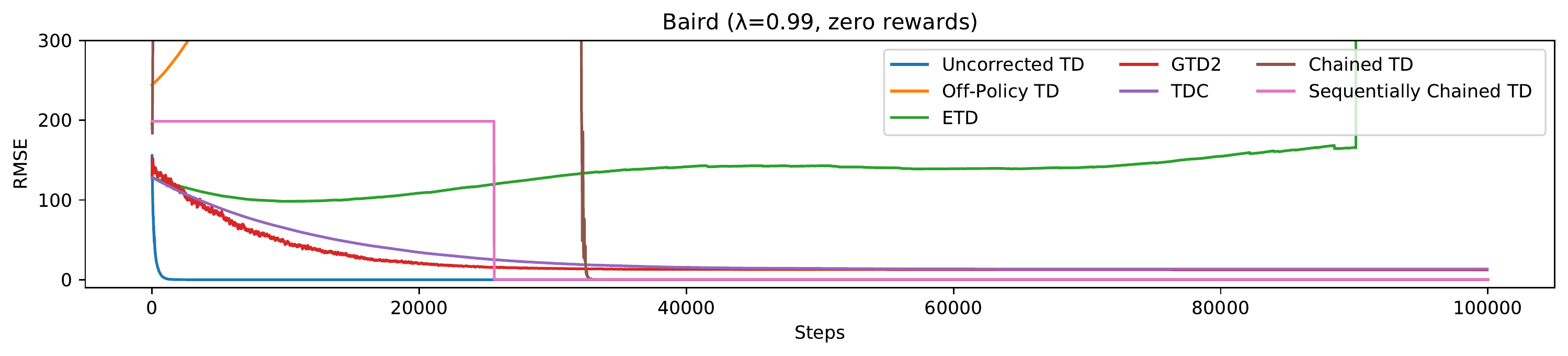}
\includegraphics[width=0.8\textwidth]{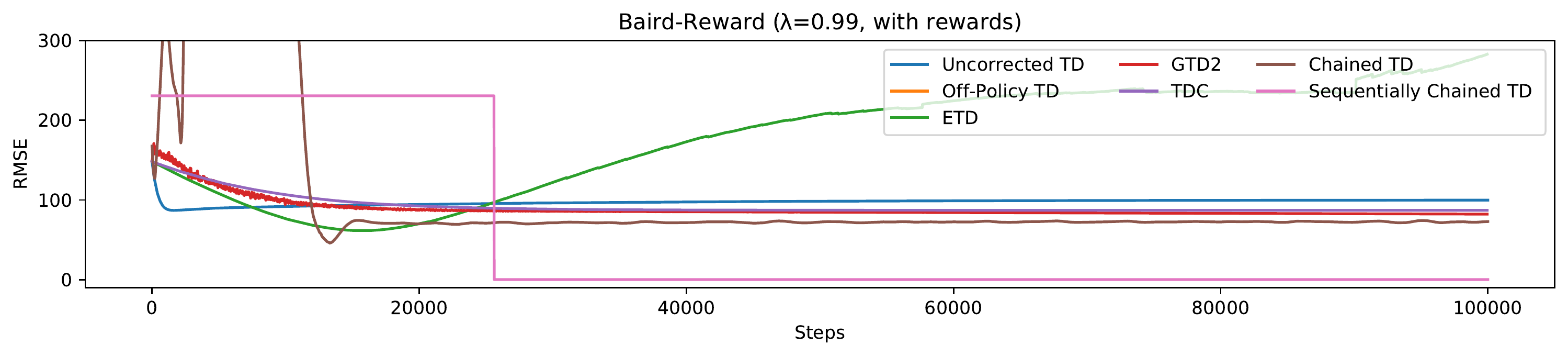}
\includegraphics[width=0.8\textwidth]{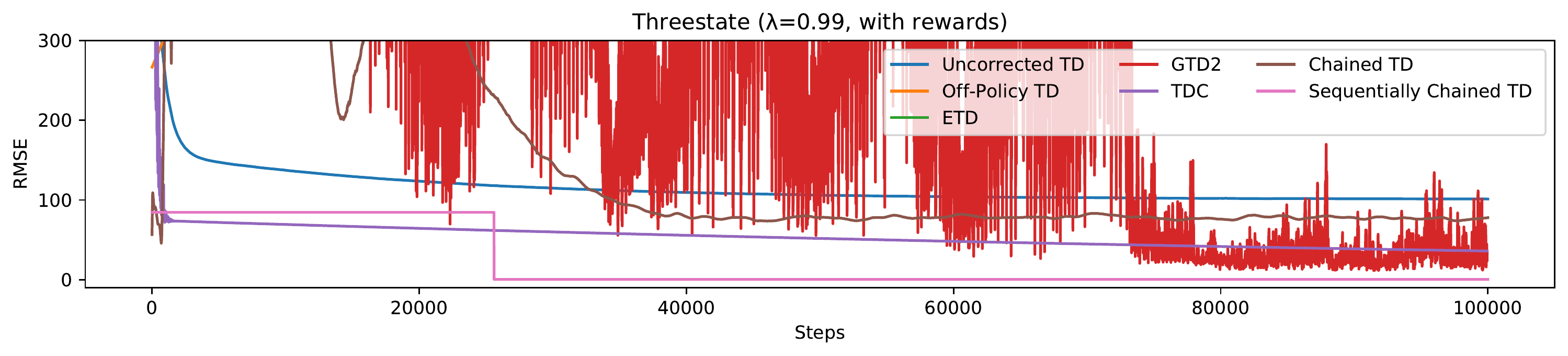}
\caption{Learning process measured in RMSE over 100 validation seeds of the 1-step TD algorithms and MDPs corresponding to Table 1 in the paper. Note that Off-Policy TD is often not visible as it diverged quickly.
Observe that MDPs with large discount and rewards (Baird-Reward and Threestate) are the most challenging and that only sequentially chained TD learning obtains RMSE close to 0. The hyper-parameters for each algorithm ($\alpha$, $\beta$ for GTD2, TDC; $\alpha$, $k$ for chains; and $\alpha$ otherwise) were selected to minimize error averaged over the final $50\%$ of training on 10 separate seeds.}
\label{fig:comparison_all_best_selected}
\end{figure*}

\end{document}